\documentclass[10pt,twocolumn]{article}

\usepackage[margin=0.72in]{geometry}
\usepackage{times}
\usepackage{microtype}
\usepackage{graphicx}
\usepackage{subcaption}
\usepackage{booktabs}
\usepackage{multirow}
\usepackage{makecell}
\usepackage{tabularx}
\usepackage{adjustbox}
\usepackage{amsmath,amssymb,amsthm,bm,mathtools}
\usepackage{algorithm}
\usepackage[noend]{algpseudocode}
\usepackage{xcolor}
\usepackage{xspace}
\usepackage{cuted}
\usepackage{stfloats}
\usepackage{placeins}
\usepackage{balance}
\usepackage{enumitem}
\usepackage{pifont}
\usepackage[numbers,sort&compress]{natbib}
\usepackage[colorlinks=true,allcolors=blue]{hyperref}
\usepackage[capitalize,nameinlink]{cleveref}
\makeatletter
\providecommand{\theHALG@line}{\thealgorithm.\arabic{ALG@line}}
\makeatother

\newcommand{\EpisodeQuerySR}{77.56}


\newcommand{\method}{INTACT\xspace}

\newcommand{\cmark}{\ding{51}}
\newcommand{\xmark}{\ding{55}}
\newcommand{\vect}[1]{\bm{#1}}
\newcommand{\R}{\mathbb{R}}
\newcommand{\best}[1]{{\bfseries\boldmath #1}}

\newtheorem{proposition}{Proposition}
\newtheorem{definition}{Definition}

\title{\vspace{-1.9em}
\texorpdfstring{%
\includegraphics[width=0.40\textwidth]{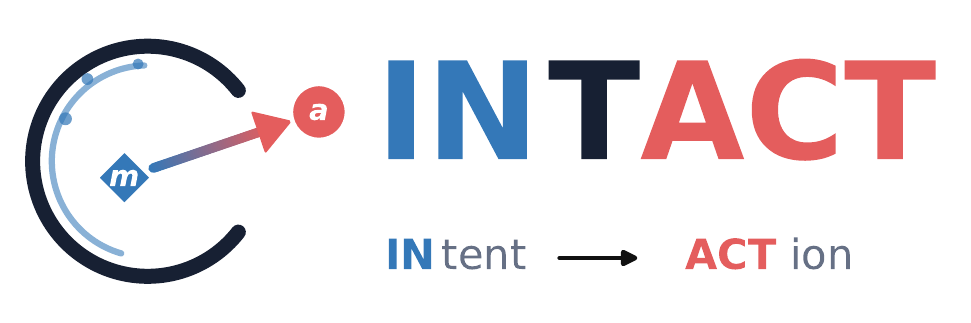}\\[-0.18em]
{\Large\bfseries Isomorphic Intent-to-Action Learning for Search-Free World Models}%
}{INTACT: Isomorphic Intent-to-Action Learning for Search-Free World Models}}

\author{
Junhan Sun$^1$\quad
Hao Zhao$^{2,4,\dagger}$\quad
Guofeng Zhang$^{1,3,\dagger}$\\[0.35em]
$^1$State Key Laboratory of CAD\&CG, Zhejiang University\\[-0.05em]
$^2$Institute for AI Industry Research (AIR), Tsinghua University\\[-0.05em]
\hspace{1.8em}$^3$InSpatio\qquad $^4$RoboParty Lab\\[-0.05em]
{\small $^\dagger$Corresponding authors}\\[0.70em]
\includegraphics[width=0.97\textwidth]{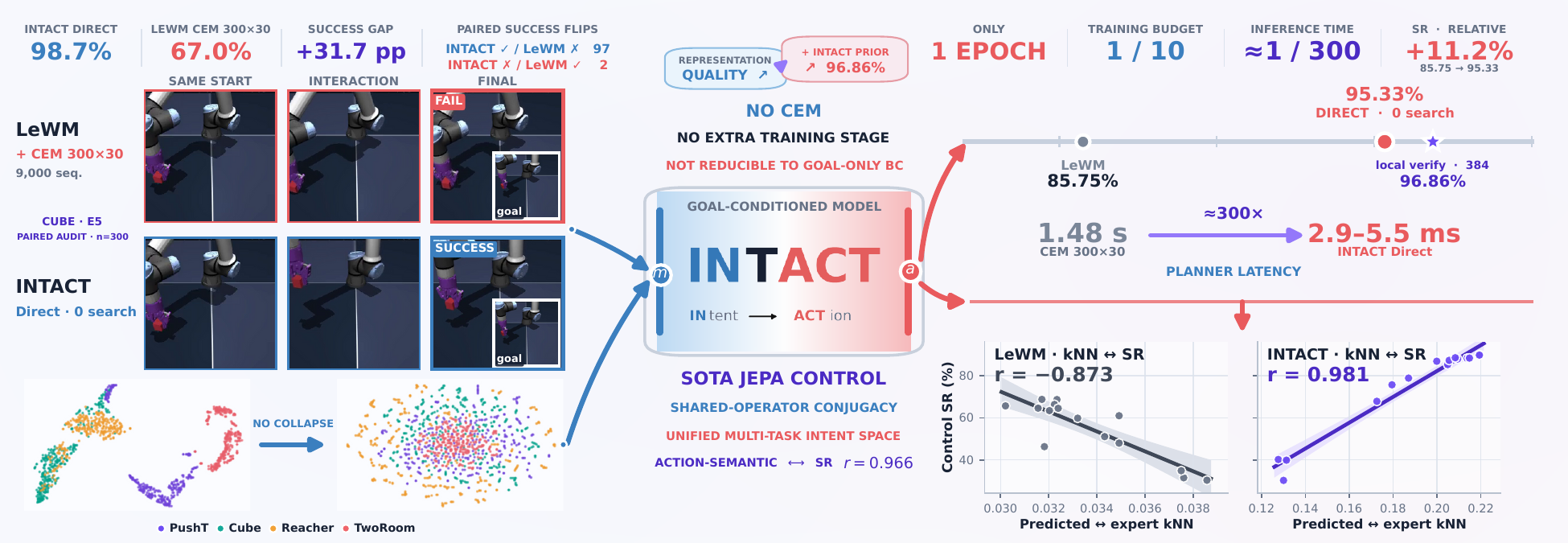}\\[0.30em]
\parbox{0.97\textwidth}{\small\textbf{INTACT turns action-aligned latent intent
into search-free control.}
Left: a fixed E5 Cube audit compares INTACT Direct with LeWM CEM
$300{\times}30$ from 300 matched starts. INTACT achieves 98.7\% versus
67.0\% success ($+31.7$ points), with 97 favorable and only two adverse
success flips. The lower panels show that INTACT prevents representational
collapse and preserves a unified multi-task intent space. Center: one shared
goal-conditioned operator maps latent motion intent directly to an action
chunk, requiring neither CEM nor an additional policy-training stage; its
joint physical and goal calls make the learned interface action-semantic
rather than goal-only behavior cloning. Right: one-epoch INTACT Direct reaches
95.33\% four-task macro SR, while local verification reaches 96.86\%.
Inference latency falls from 1.48\,s for CEM $300{\times}30$ to
2.9--5.5\,ms. Across 15 checkpoints, predicted--expert kNN alignment strongly
tracks INTACT Direct SR ($r=0.981$), whereas the corresponding LeWM--CEM
relation is negative ($r=-0.873$).}
\\[3.70em]
\parbox{0.94\textwidth}{%
\centerline{\normalsize\bfseries Abstract}
\vspace{0.44em}
\fontsize{9.55pt}{12.0pt}\selectfont
Forward latent world models predict how actions change a scene, but recover
actions for a desired change only through expensive test-time search. We
introduce \method (\textbf{IN}tent-\textbf{T}o-\textbf{ACT}ion), an end-to-end
JEPA that turns action-labeled, reward-free trajectories into a deployable
intent-to-action interface. Each transition supplies physical intent
$z_{t+1}-z_t$, while a future goal supplies deployment intent
$\operatorname{sg}(z_g)-z_t$.  The architecture is \textbf{isomorphic} in two
explicit between--and senses: \textbf{isomorphic between the local and goal
motion-intent backbone-input graphs} through an identical four-slot grammar
and shared parameters, and \textbf{isomorphic between the supported local and
goal motion-intent families} through the action-law semantics induced by the
same predictor, rather than pointwise latent equality.  INTACT also realizes
two complementary forms of \textbf{intact} transfer: \textbf{intact between RGB
evidence and latent intent coordinates}, because end-to-end action gradients
retain action-effective information while filtering nuisance unrelated to
motion intent; and \textbf{intact between latent intent families and their
corresponding action-law families}, because the shared predictor learns the
supported family mapping.  Asymmetric endpoint gradients ground physical
successors and fix future goals as anchors, joining representation learning
and control without pointwise latent matching or globally linear dynamics.
The resulting intent coordinates support a robust distributional action law:
its conditional mean serves directly as a \textbf{search-free} policy, while
sampling remains available for diversity or optional verification rather than
being required for goal-directed control.  On the four official LeWM tasks,
one-epoch, zero-search models reach 85.78\%, 100.00\%, 97.67\%, and 97.89\%
success. Optional local CEM
centered on the Direct plan reaches 96.86\% macro success using 384 instead of
9,000 candidate sequences, reducing sampling by 23.44$\times$ while improving
pure CEM by 16.00 points. One shared four-task encoder reaches 89.39\% E5
Direct macro and improves every task over jointly trained LeWM, while
predicted--expert action-family kNN tracks Direct success at $r=0.954$. Direct
inference takes 2.9--5.5\,ms.}
}

\date{}

\begin{document}
\raggedbottom
\onecolumn
\maketitle
\twocolumn

\section{Introduction}
\label{sec:introduction}

Latent world models learn a forward conditional: given a visual state and an
action, predict what happens next. Goal-conditioned control then numerically
inverts this map by proposing action sequences, rolling them through the model,
and retaining those that approach a goal
\citep{hafner2019planet,hansen2024tdmpc2,zhou2025dinowm,maes2026lewm}. This
division creates a representation--control asymmetry. Actions shape the latent
during training, yet CEM or MPPI begins deployment from generic random or
Gaussian actions \citep{deboer2005cem,williams2017mppi}. Proposals become
task-directed only after repeated model rollout, scoring, and refitting. The
world model learns what an action will do, but not which action should realize
a requested latent change. Action and latent are coupled in the forward
predictor but semantically uncalibrated in the inverse direction needed for
control.

The missing supervision is already present in offline trajectories. Our core
observation is that \textbf{even when demonstrations are unordered and training
uses no reward, every action-labelled transition still reveals a
state-conditioned motion intent and the action that realizes it}. Across many
trajectories, these labels reveal families of state-conditional motion intents
and their corresponding action families. Our insight is to represent a desired
motion by a latent intent $m_t$, train
$\pi_\eta(a_t\mid z_t,m_t)$ jointly with the world model, and let observed
actions organize which latent intents require compatible control. The same
intent construction is available at deployment, so the learned representation becomes an
action proposal interface rather than only a scoring space.

\begin{figure*}[t]
\centering
\includegraphics[width=0.99\textwidth]{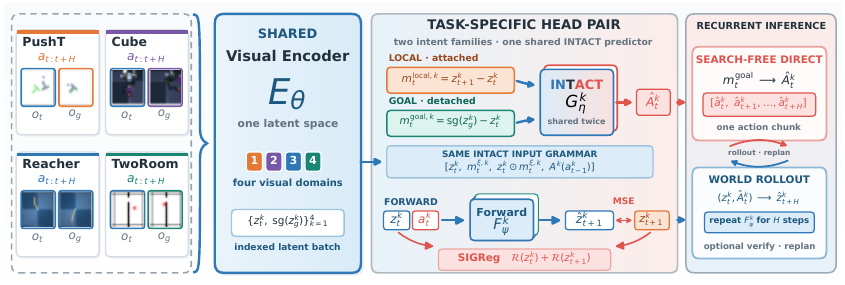}
\caption{\textbf{Shared-encoder INTACT training and recurrent control.}
Four visual domains use one encoder and task-specific predictor pairs. Within
each task, graph-isomorphic local and goal calls enter the same INTACT
Predictor through one matched input grammar, using attached local intent
$z_{t+1}^k-z_t^k$ or detached goal intent
$\operatorname{sg}(z_g^k)-z_t^k$. The Forward
Predictor retains latent-dynamics and SIGReg supervision. At inference, the
goal-conditioned call emits one action chunk for search-free Direct control;
the world rollout is retained only for optional verification and replanning.}
\label{fig:overview}
\end{figure*}

Forward prediction already gives a latent predictive meaning; calling that
latent wholly meaningless would be incorrect. What it does not specify is an
intent-to-action semantics: which compressed distinctions should be immediately
useful for converting a requested motion into control.
In this sense, a forward-only bottleneck behaves like an uncalibrated funnel:
it gathers the information needed to forecast observations while leaving its
action geometry incidental. Intent-conditioned likelihood turns that funnel
into a calibrated filter. Action-equivalent requests are organized for rapid
readout, while the forward objective retains contacts, topology, and visual
residuals needed for future prediction. Representation shaping is therefore
not a cosmetic regularizer; it determines whether the learned state can be
queried directly by a downstream controller.

The name \method encodes three complementary notions of intactness. First,
\emph{representation intactness} asks whether RGB-to-latent compression keeps
the controllable distinctions needed by the intent-to-action conditional. This
is not synonymous with higher rank, more training epochs, or retaining every
pixel detail. Second, \emph{conversion intactness} asks whether relations among
state-conditioned intent families survive their map into the corresponding
action families. Third, \emph{\textbf{isomorphic} intactness} has two levels.
Structurally, the local and goal motion-intent calls use \textbf{isomorphic}
backbone-input graphs: the same typed slots, parameter-labelled predictor, and
action-law codomain. Semantically, their supported motion-intent families are
\textbf{isomorphic} through the conditional action relation induced by that
shared predictor: each conditional quotient class corresponds bijectively to
its realizable action-law image, without requiring local and goal latents to be
pointwise equal. Within either call, the first-order intent
$m_t$ and its state-dependent interaction $z_t\odot m_t$ retain one matched
feature grammar \textbf{intact}. The latter
is a fixed state--intent bilinear feature---a low-cost second-order interaction
in the inputs---not a learned second-order dynamics model. A frozen downstream head can read only what its encoder
already exposes; it cannot reconstruct controllable distinctions discarded
during pretraining. End-to-end action likelihood therefore shapes both what
enters the latent filter and how it is converted into control. In the first
sense, action-relevant intent must survive RGB compression \textbf{intact}; in
the second, the relation between an intent family and its corresponding action
family must survive conversion \textbf{intact}.

This idea is not an additive recipe of two auxiliary losses. Taken alone, the
local branch is physical inverse dynamics; taken alone, the deployment branch
is GCSL-like goal-conditioned imitation. Neither isolated objective expresses
the coupling that INTACT needs. The underlying asymmetry is
\begin{equation}
\underbrace{z_{t+1}}_{\substack{\text{physically grounded}\\
\text{unavailable before acting}}}
\qquad\text{vs.}\qquad
\underbrace{z_g}_{\substack{\text{available before acting}\\
\text{not a one-step successor}}}.
\label{eq:intro_endpoint_asymmetry}
\end{equation}
INTACT resolves it by asking one conditional action operator to interpret both
families, with $m\in\{m^{\rm local},m^{\rm goal}\}$,
\begin{equation}
G_\eta:(z,m,a^-)
\longmapsto p_\eta(a\mid z,m,a^-).
\label{eq:intro_shared_operator}
\end{equation}
The local successor remains attached and supplies physical reachability and
representation shaping; the future goal is a stop-gradient deployment anchor.
The two conditions are never forced together by latent $L_2$, while the
Forward Predictor preserves world information richer than the immediate
action. This shared, graph-isomorphic dual call is the coupling mechanism, not
mere coexistence of two losses.

We therefore treat the deployment condition as a \emph{motion request}, not an
imagined trajectory. Our waypoint attribution uses the horizon-normalized intent
$m_{t,h}^{\rm waypoint}=(\operatorname{sg}(z_g)-z_t)/h$; the final model instead
exposes the raw goal intent
$m_t^{\rm goal}=\operatorname{sg}(z_g)-z_t$. The former imposes an equal-progress
scale, while the latter retains the complete terminal residual and lets the
nonlinear INTACT Predictor choose state-dependent progress. In both cases, the
demonstrated action at the same state supervises the intent, and the same intent
semantics are used at deployment. After their conditions are constructed, the
two calls form \textbf{isomorphic typed predictor graphs}; their upstream
gradient routes remain asymmetric. They do not pretend that a terminal goal is the next
physical frame. This supervision is obtained \emph{without reward labels or
additional semantic annotations; multi-task training uses the known domain
index only to select task-specific heads}.

The resulting geometry is conditional. At fixed current state $z$, two
endpoint conditions are equivalent when they induce the same expert action
law:
\begin{equation}
y\sim_z y'
\quad\Longleftrightarrow\quad
p_E^\star(a\mid z,y)=p_E^\star(a\mid z,y').
\label{eq:intro_quotient}
\end{equation}
Distinct goals may share the same first action before a doorway and separate
after a junction; a real successor and a goal-derived intent may be equivalent
without being close in Euclidean latent space. This quotient applies only to
the action condition. The complete JEPA latent must still preserve contacts,
obstacles, and future visual information required by forward prediction.
We call the resulting equivalence class a \textbf{conditional action quotient}.

We instantiate this view in \method (\textbf{IN}tent-\textbf{T}o-\textbf{ACT}ion
learning). INTACT comprises a \emph{Forward Predictor} and one shared
\emph{INTACT Predictor}. The latter is called with
$m_t^{\rm local}=z_{t+1}-z_t$ to preserve action-recoverable physical change and
with $m_t^{\rm goal}=\operatorname{sg}(z_g)-z_t$ to map a pre-action goal intent
to the demonstrated action block. These calls share feature grammar and
parameters, but not necessarily their numerical distributions or encoder
gradients. At inference, INTACT alternates the INTACT Predictor mean with the
unchanged Forward Predictor and replans from real observations.
Its \textsc{Direct} mode evaluates no candidates and calls no terminal latent
cost; CEM remains available as an optional verifier.

To our knowledge, INTACT is the first end-to-end latent-control architecture to
combine all three ingredients: one shared proper action likelihood for observed
physical-transition and deployable goal-intent condition families, attached
physical-successor endpoints but stop-gradient future-goal anchors, and no
pointwise loss that collapses the two latent condition families. The novelty is
the complete statistical construction rather than inverse dynamics alone. It
aligns the two families through the behavior they induce under the same
conditional action operator, not by making their latent coordinates equal.
The goal call itself is a goal-conditioned imitation objective. Our claim is
not that INTACT lies outside behavior cloning; it is that goal-only,
post-hoc frozen-representation, and independent-actor explanations are
insufficient for the measured behavior of the complete system.
Accordingly, goal-intent-only serves as the GCSL-like ablation in our
factorial; the contribution of full INTACT is the end-to-end coupling of this
deployment likelihood with JEPA dynamics and a shared local physical anchor.
Equivalently, goal-intent-only identifies one familiar branch objective;
Full INTACT begins when the attached physical and stop-gradient deployment
conditions must be explained by the same parameterized action-law operator.

This distinction also rules out a tempting shortcut. Making a physical
successor close to a hand-designed latent waypoint can make a trajectory look straighter while
discarding velocity, contact, or action-identifiable information. INTACT does
not regress a goal intent onto a successor. It uses a probabilistic action likelihood
to preserve the conditional family relation induced by the demonstrated
action.

Nor is this merely pointwise interpolation behavior cloning. The goal intent is
constructed inside the jointly learned world representation, and its likelihood
updates the current-state encoding while the future goal is stop-gradient.
This goal-intent path has no physical-successor gradient; successor-based shaping
is supplied by the separate physical inverse loss in the full objective. The
Forward Predictor remains responsible for rolling the next request. Most importantly, actor-disabled CEM improves after
local/goal-intent training, and conversion-fidelity metrics predict closed-loop SR
across objective designs. The evidence therefore separates representation
shaping, intent-to-action conversion, and optional search rather than assigning
all gains to a terminal policy head.

We deliberately separate the theory and evidence into single-task and
multi-task regimes. In one task, physical inverse and goal-intent supervision are
complementary: the inverse should improve the shared representation even when
its actor is removed, while goal intent should primarily improve the deployed
conditional. With a shared encoder across heterogeneous tasks, however,
individually valid local- and goal-intent gradients can interact. We therefore test
the minimal $2{\times}2$ factorial---neither action loss, physical inverse only,
goal intent only, or both---with one shared encoder and task-specific small heads.
This isolates whether local physical anchoring and deployable goal conditioning
remain complementary under heterogeneous dynamics.

Here ``unified'' has an operational meaning: all four domains update one visual
encoder in one training run, and the same deployment-supported action loss that
produces control also shapes that encoder. This differs from training one world
model per domain or freezing a pretrained LeWM and fitting a downstream
controller or sampling prior. The published LeWM, C-JEPA, Fast-LeWM, GC-IDM,
PRISM, and Qantara control results use task-specific checkpoints or cover only a
subset of the suite. Among the interfaces surveyed in
\cref{tab:positioning}, none previously reports this four-domain joint setting;
our novelty claim is deliberately restricted to that comparison set.

The official four-task LeWM benchmark supports both parts of the account. With
one full-data epoch and no candidate search, the final distance-preserving
single-task model reaches 85.78\%, 100.00\%, 97.67\%, and 97.89\% on PushT,
Cube, Reacher, and TwoRoom, with three training seeds per task. Earlier matched
PushT attribution shows that actor-disabled CEM $30{\times}10$ improves from
42.22\% for LeWM to 57.67\% with inverse supervision, 61.44\% with a matched
goal-intent objective, and $69.44\!\pm\!1.64$\% for the Waypoint INTACT model with
SIGReg $0.03$. Because the actor is disabled, this is functional evidence
beyond the INTACT Predictor proposal; frozen probes provide direct readout
evidence.
Direct uses no candidate search and only 2.9--5.5\,ms in our audited
planner-side measurements; matched PushT is 4.8\,ms. Relative to the
measured actor-initialized CEM $300{\times}30$ average of 1.48\,s, this is about
a $300\times$ planner-latency reduction. It removes a major model-based search
bottleneck for high-frequency control, although it is not an end-to-end VLA
latency benchmark.
Over all 12 final checkpoints, the complete 576-job planner audit finds that
Guarded A, a coherent Direct plan followed by reference-preserving
$128{\times}3$ local raw-action CEM at $\sigma_0=0.25$, reaches 96.86\% macro
and 92.22\% worst-task SR. Relative to matched pure CEM
$300{\times}30$, it gains 16.00 macro points while reducing sampled candidate
sequences from 9,000 to 384 (23.44$\times$). Ordinary $\sigma_0=0.25$ is only
0.08 macro point lower, making it a particularly simple strong default.

For multi-task learning, we use the first complete controlled Math-SDPA matrix:
six objective cells, three training seeds, and three evaluation seeds per
checkpoint. At E1, Full INTACT reaches $42.69\!\pm\!7.19$\% Direct macro SR,
versus $35.75\!\pm\!7.58$\% for matched waypoint-intent-only and
$33.47\!\pm\!3.98$\% for inverse-only. Relative to goal-intent-only, adding physical
inversion raises PushT and TwoRoom by 12.89 and 15.33 points and reduces TwoRoom
training-seed standard deviation from 31.13 to 10.58 points. At E5, one shared
Goal-displacement INTACT encoder reaches
$89.39\!\pm\!0.77$\% macro Direct SR. Relative to the matched shared-encoder
LeWM, it improves all four tasks by 5.66/32.23/12.56/42.44 points; it also
exceeds the published task-specific LeWM macro and its Cube/Reacher scores.
Against the matched three-seed goal-intent-only cell, the physical call adds
$8.78\!\pm\!1.63$ macro points. With every action head disabled, pure CEM
$300{\times}30$ rises from $66.17\!\pm\!2.67$\% for LeWM to
$70.08\!\pm\!1.13$\% for Goal-displacement INTACT; adding Guarded A to the
learned Direct plan further raises $89.39\!\pm\!0.77$\% to
$90.47\!\pm\!0.84$\%. These matched controls separate representation/Forward
usability, direct intent conversion, and optional verification.
Separately, a
45-checkpoint E1--E5 simultaneous-dual-likelihood audit supplies evidence
beyond objective rankings: predicted--expert action-family kNN and CKA have
Pearson $r=0.954$ and $0.897$ with official SR, respectively, compared with
$0.815$ for pointwise action $R^2$. After controlling epoch and cohort, kNN
remains $r=0.902$. This gap indicates that preserving local and global action-family
geometry is more informative than recovering every expert action by a single
point estimate.

Our contributions are:
\begin{enumerate}[leftmargin=*,itemsep=2pt,topsep=3pt]
    \item We define a state-conditional action quotient whose classes are
    canonically isomorphic to the realizable action-law image, represent
    endpoint requests by latent displacement coordinates, and introduce an
    end-to-end \emph{isomorphic implicit-training architecture}: one shared proper likelihood aligns observed-transition and
    deployment-intent condition families, with asymmetric gradient routing and
    without forcing their latent endpoints to coincide. A matched sharing
    ablation shows that two independent actors remain worse even at double
    capacity, directly supporting the shared-operator choice.
    \item We identify the complementary matched intra-call grammar required by
    the INTACT head: the first-order intent $m_t$ must be paired with its matched
    state--intent bilinear interaction $z_t\odot m_t$. The new A--G control
    shows that centering the product alone gives only $+0.89$ points, whereas
    matching the centered main condition and interaction yields a
    $+3.89$-point residual--product interaction with the same direction in all
    three training seeds. This is a state-conditioned inverse-control factor,
    not a second-order dynamics model.
    \item We develop a non-diffusion, autoregressive intent--forward controller
    that emits a state-conditioned action distribution rather than a single
    static answer. It supports zero-search Direct control and plan-centered
    local CEM: the best fixed $128{\times}3$ rule reaches 96.86\% macro and
    92.22\% worst-task SR while using only 4.27\% of the action-step proposals
    used by CEM $300{\times}30$.
    \item We introduce native-coordinate action-family diagnostics---predicted--expert
    CKA and kNN overlap---and show across 45 eligible E1--E5 checkpoints that
    their correlations with official SR ($r=0.897/0.954$) exceed that of
    pointwise action $R^2$ ($r=0.815$); kNN remains $r=0.902$ after controlling
    epoch and cohort, exposing
    family-level intent-to-action preservation as the relevant mechanism. A
    21,600-episode gauge intervention further shows that paired correspondence
    recovers 68.04\% SR versus 9.46\% after pair shuffling, while a Full
    backbone cannot give a LeWM actor an untrained deployment map (5.08\%).
    \item We show that joint intent supervision improves not only the deployed
    actor but also the shared representation: after removing the INTACT
    Predictor at evaluation, matched actor-disabled pure CEM still improves
    over LeWM, with frozen probes independently confirming more recoverable
    state and intervention information.
    \item We give latent space an operational semantics by extending a
    Gaussian-regularized self-supervised predictor into a multi-task mapping
    between motion-intent families and action families. This principle is not
    specific to LeWM's encoder--encoder JEPA: it applies to general learned
    representations whenever action-labelled transitions are available.
\end{enumerate}

\section{Related Work}
\label{sec:related}

\paragraph{Latent world models and JEPAs.}
PlaNet and Dreamer learn recurrent latent dynamics through generative
objectives \citep{hafner2019planet,hafner2023dreamerv3}; TD-MPC2 learns compact
dynamics jointly with control-relevant predictions \citep{hansen2024tdmpc2}.
JEPA-style world models instead predict in representation space. DINO-WM plans
with a frozen DINOv2 encoder \citep{zhou2025dinowm}, PLDM trains end-to-end with
a VICReg-style objective \citep{sobal2025pldm}, and LeWM combines latent
prediction with Sketched Isotropic Gaussian Regularization (SIGReg) for stable
end-to-end learning \citep{balestriero2025lejepa,maes2026lewm}. We retain
LeWM's predictive backbone but ask how an action-conditioned deployment
interface can shape, rather than collapse, that same representation.
Two recent extensions improve the forward-model side without changing this
search interface. C-JEPA uses object-level latent masking and reports efficient
CEM on PushT \citep{nam2026cjepa}; Fast-LeWM predicts all action-prefix horizons
in parallel and reduces the matched TwoRoom CEM solve from 54.4 to 28.3 seconds
while retaining the same candidate budget \citep{gao2026fastlewm}. Both train a
separate world model per reported control domain. INTACT instead asks whether a
single shared encoder can expose a directly deployable action conditional
across the four LeWM domains.
The distinction is important: a forward-only encoder is not meaningless, but
its compression criterion is prediction sufficiency rather than immediate
action sufficiency. INTACT uses labeled transitions to calibrate that
bottleneck without replacing its predictive role.

\paragraph{Inverse dynamics and representation learning.}
Inverse models have long been used to isolate controllable visual structure
\citep{pathak2017curiosity,efroni2022exogenous} and support multitask imitation
\citep{brandfonbrener2023inverse}. An early robotic precursor jointly learned
forward and inverse models from unlabeled poking interaction, using inverse
prediction to shape visual features and forward prediction to regularize their
dynamics \citep{agrawal2016poking}. SMWM shows that one-step inverse supervision
can regularize an end-to-end JEPA toward sensorimotor features
\citep{ivashkov2026smwm}; broader probing work likewise links temporal
pretraining to action recoverability \citep{yeom2026actionrelevant}. These
results concern action-relevant representation, but a local inverse evaluated
only on observed successors does not identify its behavior on a goal-derived
intent. \method makes this support mismatch explicit and keeps the learned
conditional as a deployment interface.

\paragraph{Hindsight goal-conditioned imitation.}
GCSL relabels a trajectory with a future state that it actually reaches and
maximizes the corresponding goal-conditioned action likelihood, iterating
between collection under the current policy and supervised updates
\citep{ghosh2019gcsl}. Complementary formulations cast goal reaching as
recursive future-state classification or learn a contrastive representation
whose inner product is a goal-conditioned value function
\citep{eysenbach2021clearning,eysenbach2022contrastive}. INTACT's deployment
branch has this GCSL-like statistical
form, and goal-intent-only is the corresponding hindsight-imitation ablation in
our factorial. It is not the original GCSL algorithm: it retains the JEPA
forward objective and SIGReg and uses fixed offline expert trajectories rather
than iteratively collecting its own. INTACT is not goal-only GCSL. Its
deployment branch is a GCSL-like hindsight action likelihood inside a jointly
trained JEPA, while the complete method additionally uses the same conditional
operator on attached one-chunk physical successors to anchor local
reachability.
This distinction is structural rather than terminological. Goal-intent-only
captures one branch call and is therefore the GCSL-like baseline reported in
our tables; physical-inverse-only captures the other branch call. Only Full
INTACT evaluates both on each transition through one weight-shared,
graph-isomorphic predictor while routing endpoint gradients asymmetrically.
Thus the shared operator is not cosmetic parameter tying: it is the mechanism
that gives physically grounded and deployable condition families a common
action-law semantics without matching their latent coordinates.

\paragraph{Options, empowerment, and unsupervised skill discovery.}
The options framework treats closed-loop policies as temporally extended action
units \citep{sutton1999options}, while empowerment measures the information
capacity from an agent's action channel to reachable outcomes
\citep{klyubin2005empowerment}. Modern unsupervised RL often combines these
ideas with a latent skill variable: DIAYN makes skills distinguishable from
visited states \citep{eysenbach2019diayn}; DADS favors skills with predictable
dynamics \citep{sharma2020dads}; CIC contrastively aligns skills with state
transitions \citep{laskin2022cic}; and METRA learns a temporally grounded metric
space in which a policy moves along latent directions \citep{park2024metra}.
These methods share the useful interface $\pi(a\mid s,\xi)$, but solve a
different statistical problem. Their code $\xi$ is generally sampled or
discovered through interaction and optimized by intrinsic reward, mutual
information, coverage, or temporal distance. Our $m$ is not a free behavior
index: it is a state-conditional motion intent grounded in an offline
demonstration, and the recorded action is its supervised target. Rather than
discovering behaviors that make latent codes identifiable, we identify which
latent requests are equivalent under the demonstrated action law.

\paragraph{Metric learning and action-induced equivalence.}
Classical metric learning uses pair or triplet labels to make Euclidean
embeddings invariant within a semantic class and separated across classes
\citep{hadsell2006drlim,schroff2015facenet}. In control, bisimulation metrics
learn task-relevant invariances without pixel reconstruction, while
goal-conditioned bisimulation extends this idea to functional equivariance
between analogous goals
\citep{zhang2021invariant,hansenestruch2022bisimulation}. Action labels play an analogous
organizational role here: across otherwise unordered trajectories, they reveal
families of motion intents requiring similar control. The distinction is that
\method does not impose a pairwise distance or collapse complete JEPA latents
with the same action. Equivalence is defined at fixed state by equality of the
conditional action law and learned through a proper action likelihood. Forward
prediction and SIGReg preserve variations within an action fiber that remain
important for future dynamics. In this sense our quotient is an
output-distribution geometry, not a metric-learning loss on observations.

\paragraph{Amortized and prior-guided world-model control.}
GLAMOR learns goal-conditioned inverse and prior models to guide random
shooting \citep{paster2020glamor}. Biased-MPPI provides a general mechanism for
fusing ancillary sampling distributions into MPPI
\citep{trevisan2024biasedmppi}. More recently, PRISM predicts a
current/goal-conditioned Gaussian trajectory prior from a frozen JEPA encoder
and fuses it with MPPI by precision-weighted Product-of-Gaussians updates
\citep{wang2026prism}. It addresses a train--deployment goal shift through the
planner's deployment-time cost, but still evaluates sampled trajectories; its
reported limitations include task-specific near-expert demonstrations and a
unimodal local prior for potentially multimodal long-horizon behavior.

GC-IDM first studies a pairwise inverse over consecutive frozen LeWM latents
with plans produced by latent linear interpolation
\citep{nguyen2026gcidm}. Its oracle action $R^2=0.993$ does not rescue invalid
straight latent paths; on TwoRoom, its appendix reports 34\% for interpolation,
48\% after refinement and candidate selection, and 87\% for CEM. Conversely,
its final goal-conditioned controller reaches 100\% TwoRoom SR with training
$R^2$ around 0.20 by predicting useful directions and replanning. GC-IDM thus
already shows that scalar coordinate recovery is not a sufficient theory of
control. It avoids imagined latent paths in its final controller and conditions
on encoded current/goal observations plus remaining horizon. INTACT takes
another route: the waypoint intent is a request coordinate rather than a physical
successor, and is placed on training support by pairing it with the demonstrated
current action. INTACT also updates the JEPA encoder end to end and evaluates
both search-free execution and the actor-disabled world model. Our claim is not
the first inverse controller or learned sampling prior; it is that
deployment-supported intent supervision can organize the representation and
directly amortize action search.

Statistically, the frozen-encoder stage used by GC-IDM and PRISM is closer to
behavior cloning or prior fitting on top of a fixed visual representation: it
can exploit information that the encoder already exposes, but its action loss
cannot decide which distinctions that encoder should preserve. INTACT is
unified in a stricter sense. The visual encoder, forward dynamics, and
deployment-supported action conditional are optimized together, and the same
encoder is shared across all four domains. Among LeWM and the LeWM-derived
systems compared here, we are not aware of a prior report that jointly trains
one four-domain encoder while allowing the deployed action conditional to shape
that encoder. This is a scoped claim about the published interfaces in
\cref{tab:positioning}, not a claim that frozen-feature controllers are
ineffective.

\paragraph{Action quotients and visual fibers.}
QuoVLA formalizes a quotient of VLM prompt representations induced by the
Bayes-optimal action-trajectory law and approximates it with quantization and a
dual-branch VLA design \citep{wang2026quovla}. We share the general principle
that representations distinguished by perception may be equivalent for
action, and do not claim to originate action quotients. Our quotient acts on
local/goal motion \emph{intents at fixed JEPA state}, not on the complete
VLM prompt representation; it predicts a transition-aligned action block and
coexists with an explicit forward world model. FiberTune identifies the
complementary risk that action supervision leaves visually meaningful
directions within an action fiber unconstrained and therefore vulnerable to
collapse \citep{lin2026fibertune}. This supports our architectural separation:
the actor may quotient action-equivalent conditions, while forward prediction
and SIGReg preserve the richer world latent.

\paragraph{Unified video--action models.}
UWM couples video and action diffusion to expose forward, inverse, policy, and
generation queries \citep{zhu2025uwm}; DreamGen converts imagined futures to
actions through a separate inverse or latent-action model
\citep{jang2025dreamgen}. Qantara trains a JEPA for several dispatch-dependent
control paradigms using bridge flow \citep{rakhimov2026qantara}. These systems
motivate a unified conditional view. Our scope is deliberately smaller and
mechanistic: a compact reward-free JEPA, a proper endpoint action likelihood,
and controlled evidence linking the statistical unit of supervision to latent
geometry and direct control.

\begin{table*}[t]
\centering
\scriptsize
\caption{Positioning among closely related methods. ``Action$\rightarrow E$''
means that action supervision updates the visual/world encoder, and
``Shared-4T $E$'' means one encoder is jointly trained over all four official
LeWM domains.  The final column states each method's actual training-to-control
path; a cross denotes an unreported interface, not an impossible extension.}
\label{tab:positioning}
\setlength{\tabcolsep}{2.4pt}
\renewcommand{\arraystretch}{1.08}
\begin{tabularx}{\textwidth}{@{}lcc
>{\raggedright\arraybackslash\hspace{0pt}}p{1.48cm}
>{\raggedright\arraybackslash\hspace{0pt}}p{1.62cm}
cc>{\raggedright\arraybackslash\hspace{0pt}}X@{}}
\toprule
Method & Action$\to E$ & Shared-4T $E$ & Condition & Action obj. & Direct & Search & Training-to-control path \\
\midrule
DINO-WM \citep{zhou2025dinowm}
& \xmark & \xmark & action & none & \xmark & \cmark & Frozen image encoder; CEM ranks latent rollouts \\
C-JEPA \citep{nam2026cjepa}
& \xmark & \xmark & action/proprio. & none & \xmark & \cmark & Frozen object features; CEM searches masked rollouts \\
LeWM \citep{maes2026lewm}
& \xmark & \xmark & action & none & \xmark & \cmark & Joint JEPA prediction; CEM inverts the forward model \\
Fast-LeWM \citep{gao2026fastlewm}
& \xmark & \xmark & action prefix & none & \xmark & \cmark & LeWM training; parallel action prefixes are searched \\
SMWM \citep{ivashkov2026smwm}
& \cmark & \xmark & successor & local action & \xmark & \cmark & Local inverse shapes the encoder; deployment still uses CEM \\
GC-IDM \citep{nguyen2026gcidm}
& \xmark & \xmark & goal + horizon & next action & \cmark & \xmark & Frozen LeWM; a separate goal/horizon actor acts directly \\
PRISM \citep{wang2026prism}
& \xmark & \xmark & state/goal & Gaussian prior & \xmark & \cmark & Frozen LeWM; a separately trained prior initializes MPPI \\
Qantara \citep{rakhimov2026qantara}
& \cmark & \xmark & dispatch-dependent & action flow & \cmark & \cmark & Joint JEPA; bridge-flow heads serve dispatch modes \\
QuoVLA \citep{wang2026quovla}
& \cmark & \xmark & visual/language & action chunk & \cmark & \xmark & VLM action quotient; quantized intent drives a VLA actor \\
\method{}
& \cmark & \cmark & local/goal intent & action block & \cmark & opt. & One shared JEPA/actor law; Direct by default, search verifies \\
\bottomrule
\end{tabularx}
\end{table*}

\section{INTACT: Theory and Method}
\label{sec:method}

INTACT augments a forward JEPA with one conditional action operator used twice:
once on a realized physical transition and once on a condition available before
acting.  The two calls share parameters and feature grammar but deliberately
have different endpoint gradients.  We first give the single-task construction
and then extend it to a shared multi-task encoder.

\begin{figure*}[t]
\centering
\includegraphics[width=0.98\textwidth]{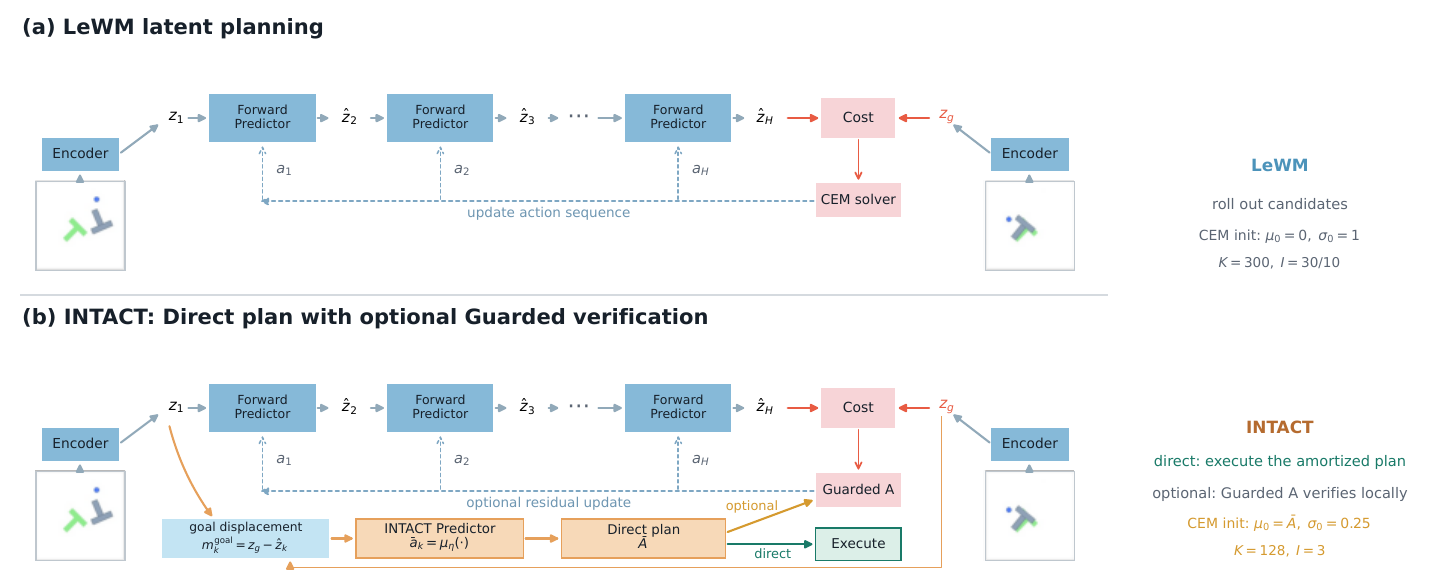}
\caption{\textbf{INTACT converts mandatory search into optional verification.}
LeWM samples raw actions from an uninformed Gaussian and uses forward rollouts
to become goal-directed.  INTACT alternates the raw goal displacement, the
shared conditional action mean, and the unchanged Forward Predictor to produce
a Direct plan.  It can execute immediately or locally verify that plan with
$K=128$, $I=3$, and $\sigma_0=0.25$.}
\label{fig:planning_pipeline}
\end{figure*}

\subsection{Problem Setting and the Forward-Action Gap}
\Cref{fig:overview,fig:planning_pipeline} show the same system from its training
and deployment views.  In \cref{fig:overview}, superscript $k$ indexes a visual
domain and its task-specific heads; below we rename that index $\tau$ and
reserve $k$ for temporal position inside a window.  For a single task we omit
$\tau$.  This removes a notational collision between the four colored task
routes in \cref{fig:overview} and the recurrent planning steps in
\cref{fig:planning_pipeline}.

For domain $\tau$, let the offline data be pixel--action trajectories
$\mathcal D^\tau=\{(\vect o_{i,s}^\tau,\vect u_{i,s}^\tau)\}$, where $i$ indexes
trajectories and $s$ indexes environment steps.  We index model-chunk
boundaries by $t$ and form the fixed action block
$\vect a_{i,t}^\tau=[\vect u_{i,Bt}^\tau;\ldots;
\vect u_{i,Bt+B-1}^\tau]\in\R^{Bd_a}$ with $B=5$.  A training window spans
$H=5$ such chunks from $o_t$ to $o_{t+H}$, and the shared encoder-projector
gives $z_j=E_\theta(o_j)\in\R^d$ and $z_g=z_{t+H}$.  At teacher step
$k\in\{0,\ldots,H-1\}$, the orange local call in \cref{fig:overview} uses
$z_{t+k+1}-z_{t+k}$, while the green goal call uses
$\operatorname{sg}(z_g)-z_{t+k}$.  They enter the same four typed input slots
of $G_\eta$ and both predict the fixed one-chunk target $a_{t+k}$; the
$\widehat A_t^\tau$ box in the figure denotes this conditional action output.
The goal endpoint stays fixed while its remaining offset is
$r=H-k\in\{5,4,3,2,1\}$ chunks.  Thus ``multi-horizon'' means multiple
remaining goal offsets, not multiple output chunk lengths, and the final raw
goal policy does not receive $r$ explicitly.

At deployment, the lower route of \cref{fig:planning_pipeline} reuses exactly
these symbols: encode the current and goal images as $z_t$ and $z_g$, set
$\widetilde z_0=z_t$, then recurrently compute
$m_k=z_g-\widetilde z_k$, $\bar a_k=\mu_\eta(\widetilde z_k,m_k,\bar a_{k-1})$,
and $\widetilde z_{k+1}=F_\phi(\widetilde z_k,\bar a_k)$.  The one-chunk means
form the plan $\bar A=(\bar a_0,\ldots,\bar a_{H-1})$ corresponding to the
training-time action output in \cref{fig:overview}.  Direct executes its first
block and replans from the next observation; Guarded A only verifies this plan
locally.  The upper route in \cref{fig:planning_pipeline} instead initializes
raw actions independently of these learned intent coordinates.

Across both routes, the unchanged LeWM Forward Predictor consumes the last
$N=3$ latents and action blocks,
\begin{equation}
\widehat z_{t+1}=F_\phi(z_{t-N+1:t},A_\omega(a_{t-N+1:t})),
\label{eq:forward}
\end{equation}
and is trained with prediction plus SIGReg,
\begin{equation}
\mathcal L_{\rm world}
=\mathbb E[\|\widehat z_{t+1}-z_{t+1}\|_2^2/d]
+\lambda_{\rm sig}\mathcal R_{\rm SIG}.
\label{eq:worldloss}
\end{equation}
This learns how a supplied action changes the latent, but not how to obtain an
action from a deployable latent request.  LeWM consequently samples actions
and uses the forward model only to rank them.

\subsection{Action-Aligned Intent Coordinates}
\label{sec:single_theory}
The center of \cref{fig:overview} gives the graphical counterpart of this
section: its attached local path carries a realized one-chunk successor, while
its stop-gradient goal path carries a condition available before acting.  In
\cref{fig:planning_pipeline}, only the latter can enter the lower deployment
route before the next action is chosen.  The endpoint $y$, displacement
$m_z(y)$, and conditional law below formalize those two pictured paths and
explain why neither one alone closes the forward--action gap.

For current latent $z$ and endpoint condition $y$, let
$p_E^\star(a\mid z,y)$ be the expert action law.
\begin{definition}[Conditional action equivalence]
\label{def:quotient}
\begin{equation}
y\sim_z y'\quad\Longleftrightarrow\quad
p_E^\star(a\mid z,y)=p_E^\star(a\mid z,y').
\label{eq:quotient}
\end{equation}
\end{definition}
The quotient $\mathcal Y/{\sim_z}$ is in one-to-one correspondence with the
realizable action-law image at $z$.
\begin{proposition}[Quotient-to-law identification]
\label{prop:quotient_image}
$\Phi_z([y]_z)=p_E^\star(\cdot\mid z,y)$ is well defined and bijective from
$\mathcal Y/{\sim_z}$ onto
$\{p_E^\star(\cdot\mid z,y):y\in\mathcal Y\}$.
\end{proposition}
The displacement
\begin{equation}
m_z(y)=y-z
\label{eq:motion_intent}
\end{equation}
is a convenient motion-intent coordinate, not a claim that latent dynamics are
globally linear or that equivalent endpoints must coincide.

A physical inverse likelihood,
\begin{equation}
\mathcal L_{z^+}=\mathbb E[-\log p_\eta(a_t\mid z_t,z_{t+1},a_{t-1})],
\label{eq:inverse_nll}
\end{equation}
asks the encoder to preserve action-recoverable transitions, but it is circular
at deployment because $z_{t+1}$ is observed only after acting.  Conversely, a
goal-conditioned likelihood is deployable but, by itself, is GCSL-like
hindsight imitation without an attached one-chunk physical anchor.  This is the
central asymmetry:
\[
\underbrace{z_{t+1}}_{\substack{\text{physical but unavailable}\\
\text{before acting}}}
\qquad\text{versus}\qquad
\underbrace{z_g}_{\substack{\text{available but not a}\\
\text{one-step successor}}}.
\]

INTACT uses one nonlinear conditional operator for both.  This is the single
INTACT Predictor drawn twice in \cref{fig:overview}; its goal invocation is the
action-generating operator on the lower route of \cref{fig:planning_pipeline}:
\begin{equation}
G_\eta:(z,m,a_{t-1})\longmapsto p_\eta(a\mid z,m,a_{t-1}),
\label{eq:homologous_predictors}
\end{equation}
while retaining the separate Forward Predictor $F_\phi:(z,a)\mapsto z^+$.
The final goal-displacement route constructs
\begin{equation}
\boxed{m_t^{\rm local}=z_{t+1}-z_t,\qquad
m_t^{\rm goal}=\operatorname{sg}(z_g)-z_t.}
\label{eq:local_goal_conditions}
\end{equation}
The local endpoint remains attached; stop-gradient applies only to the future
goal occurrence in the deployment loss.  The same $z_g$ is therefore plastic
when it appears as a real successor and fixed when it serves as an intent
anchor.  Current-state latents, $G_\eta$, and the independent forward/physical
branches continue to receive gradients.

For route $\chi$, the physical and deployable condition families are
\begin{align}
\mathcal C_{\rm phy}^{\chi}&=\{(z_t,c_t^{\rm phy,\chi},a_{t-1})\},\\
\mathcal C_{\rm dep}^{\chi}&=\{(z_t,c_t^{\rm dep,\chi},a_{t-1})\},
\label{eq:condition_families}
\end{align}
and one shared action operator serves their union,
\begin{equation}
G_\eta:\mathcal C_{\rm phy}^{\chi}\cup\mathcal C_{\rm dep}^{\chi}
\rightarrow\mathcal P(\mathcal A).
\label{eq:shared_action_operator}
\end{equation}
Its paired proper likelihood is
\begin{align}
\mathcal L_{\rm I2A}^{\chi}
={}&\lambda_{\rm inv}\mathbb E[-\log p_\eta(a_t\mid z_t,c_t^{\rm phy,\chi},a_{t-1})]\notag\\
&+\lambda_{\rm goal}\mathbb E[-\log p_\eta(a_t\mid z_t,c_t^{\rm dep,\chi},a_{t-1})].
\label{eq:family_action_alignment}
\end{align}
The first NLL is the attached physical call in \cref{fig:overview}, and the
second is its stop-gradient deployment-goal call.  Thus
\cref{eq:condition_families,eq:shared_action_operator,eq:family_action_alignment}
are the set, operator, and loss views of the same two-arrow construction in
the figure.
No term minimizes $\|m_t^{\rm goal}-m_t^{\rm local}\|^2$.  The two calls may
have different input scales and output Gaussians.  Here \textbf{isomorphic}
again has two levels: their backbone-input graphs are \textbf{isomorphic} by
typed slots and shared parameters, while their supported motion-intent
families are \textbf{isomorphic} through the conditional action-law relation
represented by $G_\eta$.  Neither level asserts numerical latent equality.

\begin{proposition}[Off-support non-identifiability]
\label{prop:nonidentifiability}
If inverse risk is evaluated only on realized successors, any two actors that
agree on that support have identical risk even when they disagree on a
deployable goal condition outside it.
\end{proposition}
\begin{proposition}[Supported quotient recovery]
\label{prop:proper}
For a well-specified conditional actor, population proper-NLL minimization
recovers $p_E^\star(\cdot\mid z,y)$ on every supported condition.  Hence
action-equivalent representatives receive the same predicted action law.
\end{proposition}
Proofs, the horizon-normalized waypoint derivation, the Bernoulli endpoint
variant, and the direct-linearity negative control are collected in
\cref{app:proofs,app:method_derivations}.  The waypoint request is an ablation;
the headline model uses the raw displacement in
\cref{eq:local_goal_conditions} because it preserves direction and remaining
distance without imposing equal progress.

The single-task objective is
\begin{equation}
\boxed{\mathcal L_{\rm ST}=\mathcal L_{\rm world}
+\lambda_{\rm inv}\mathcal L_{z^+}
+\lambda_{\rm goal}\mathcal L_{\rm goal}.}
\label{eq:objective}
\end{equation}
This end-to-end coupling makes action-law aliasing visible to the encoder: a
frozen head cannot recover a controllable distinction already collapsed by
$E_\theta$.

\begin{algorithm}[t]
\caption{INTACT training}
\label{alg:train}
\begin{algorithmic}[1]
\Require window $(o_{t:t+H},a_{t-1:t+H-1})$; shared $E_\theta,F_\phi,G_\eta$
\State $z_{t:t+H}\gets E_\theta(o_{t:t+H})$;
$z_g^{\rm anchor}\gets\operatorname{sg}(z_{t+H})$
\For{$k=0,\ldots,H-1$}
  \State $j\gets t+k$;
  $(m_j^{\rm local},m_j^{\rm goal})\gets
  (z_{j+1}-z_j,z_g^{\rm anchor}-z_j)$
  \State accumulate physical and goal NLLs through the same $G_\eta$
\EndFor
\State update $(\theta,\phi,\eta)$ with
$\mathcal L_{\rm world}+\lambda_{\rm inv}\mathcal L_{z^+}
+\lambda_{\rm goal}\mathcal L_{\rm goal}$
\end{algorithmic}
\end{algorithm}

\subsection{INTACT Predictor and Control}
\Cref{fig:overview} depicts the identical input wiring that makes the two
INTACT Predictor invocations structurally isomorphic; the four incoming
quantities are written explicitly as
\begin{equation}
h_t^\Delta(m_t)=[z_t;m_t;z_t\odot m_t;A_\omega(a_{t-1})].
\label{eq:actor_features}
\end{equation}
The product is a cheap state--intent interaction:
\begin{equation}
w^\top(z_t\odot m_t)=z_t^\top\operatorname{diag}(w)m_t,
\label{eq:product_interaction}
\end{equation}
allowing the same displacement to be interpreted differently across contact,
pose, and obstacle states.  A three-layer MLP outputs a diagonal Gaussian with
log standard deviation clipped to $[-5,2]$.  Detailed endpoint grammar and the
matched A--G product ablation are in \cref{app:design_selection}.

At deployment, the lower INTACT route in \cref{fig:planning_pipeline} is the
recurrent composition of \cref{eq:actor_features} with the unchanged Forward
Predictor: the Gaussian mean generates an $H$-block plan and
re-encodes after executing the next real action block.  Direct uses no sampled
candidate and no terminal cost.  Optional Guarded A centers a small
$128\times3$ raw-action search at this plan, retains the Direct reference and
the globally best candidate, and uses $\sigma_0=0.25$.  Pure CEM disables the
INTACT Predictor and recovers the upper search-first route in the figure; it is
retained only as a compatibility and actor-disabled control.
\Cref{fig:planning_pipeline,alg:control} therefore give the visual and
algorithmic views of the same Direct controller and optional local
verification.  The complete verifier matrix appears in
\cref{app:planner_matrix}.

\begin{algorithm}[t]
\caption{INTACT Direct control with optional local verification}
\label{alg:control}
\begin{algorithmic}[1]
\Require latent history, goal, and inference mode $v$
\State $\widetilde z_0\gets z_t$
\For{$k=0,\ldots,H-1$}
  \State $m_k\gets z_g-\widetilde z_k$;
  $\bar a_k\gets\mu_\eta(\widetilde z_k,m_k,\bar a_{k-1})$
  \State $\widetilde z_{k+1}\gets F_\phi(\widetilde z_k,\bar a_k)$
\EndFor
\If{$v=\mathrm{Direct}$} \State \Return $\bar A$ with zero candidates and cost calls \EndIf
\If{$v=\mathrm{one\mbox{-}shot}$}
  \State score $\bar A$ and clipped local perturbations; return the observed minimum
\Else
  \State initialize $(\mu,A_{\rm best})\gets(\bar A,\bar A)$ and $\sigma\gets0.25$
  \For{$i=1,\ldots,3$}
    \State score 128 samples plus $\bar A,A_{\rm best}$; refit on 16 elites
  \EndFor
  \State re-score $\mu$; return the best observed reference or candidate
\EndIf
\end{algorithmic}
\end{algorithm}

\subsection{Multi-Task Extension}
\label{sec:multitask_theory}
For heterogeneous tasks $\tau\in\{1,\ldots,T\}$, INTACT shares the encoder and
projector but keeps task-specific Forward, action-embedding, and INTACT heads.
The known domain index selects a head; no reward or additional semantic label
is used.  The objective is
\begin{equation}
\mathcal L_{\rm MT}=\sum_{\tau=1}^{T}\rho_\tau
\left[\mathcal L_{\rm world}^{\tau}
+\lambda_{\rm inv}\mathcal L_{z^+}^{\tau}
+\lambda_{\rm goal}\mathcal L_{\rm goal}^{\tau}\right].
\label{eq:multitask_objective}
\end{equation}
Every task therefore shapes the same visual coordinates without forcing
heterogeneous actions into one padded space.  The controlled factorial compares
forward-only LeWM, physical inverse-only, intent-only, and Full INTACT under
matched deterministic Math-SDPA.  We test both waypoint and raw-goal intent
coordinates.  Direct SR measures the deployed conditional; actor-disabled CEM
tests the encoder--Forward stack under a fixed optimizer; frozen probes and
predicted--expert CKA/kNN diagnose representation and family conversion.  None
of these diagnostics replaces closed-loop SR.

\section{Experimental Design}
\label{sec:experiments}

\paragraph{Tasks and evaluation.}
We use the four official LeWM domains: PushT, OGBench Cube, DMC Reacher, and
TwoRoom \citep{maes2026lewm,park2025ogbench}.  Models receive RGB observations
and five-step action blocks; simulator state is used only by probes and success
predicates.  Headline SR follows the official evaluator for comparability.
Unless stated otherwise, every checkpoint is evaluated with seeds
$\{0,1,42\}$ and 100 episodes per seed.  We separately report the public
CLEAR-LeWM v0.5.1 Moderate audit \citep{sun2026clearlewm}, which removes
initially solved pairs and repairs known Reacher and TwoRoom evaluator defects.
Official and CLEAR scores are never pooled.

\subsection{Single-Task Protocol}
\label{sec:single_experiments}
Each task-specific model is trained from scratch on its complete dataset for
one epoch with batch size 256, AdamW learning rate $5\times10^{-4}$, forward
weight 1.0, inverse weight 0.1, goal weight 0.05, and three training seeds.  The
final goal-displacement model uses SIGReg 0.02 and the shared grammar in
\cref{eq:actor_features}.  A matched waypoint chain isolates forward-only
LeWM, physical inverse supervision, deployment-intent supervision, and the
low-SIGReg setting.  Actor sharing is tested against parameter-matched,
double-capacity, two-output, and condition-token alternatives.

Direct asserts zero candidate sequences and zero terminal-cost calls.  Pure
CEM disables the INTACT Predictor and evaluates the encoder--Forward stack
under a fixed raw-action optimizer; it is a functional control, not an
encoder-only statistic.  Guarded A evaluates 128 sequences for three
iterations around the Direct plan.  The complete 576-job planner matrix,
coordinate grammar, SIGReg, straightening, optimizer, and search-radius
selection studies are deferred to
\cref{app:design_selection,app:planner_matrix} because they select a final
configuration rather than establish the principal result.

Frozen probes use episode-disjoint train/test episodes.  We measure physical
state and transition-action $R^2$, effective rank, per-dimension standard
deviation, collapsed-dimension fraction, and mean cosine similarity.  Together
with actor-disabled CEM, these distinguish representation shaping from the
benefit of executing the learned action conditional.

\subsection{Multi-Task Protocol}
\label{sec:multitask_experiments}
One distributed job traverses all four datasets.  The ViT-Tiny/14 encoder,
projector, and 192-dimensional latent are shared; each task retains a small
Forward Predictor, action embedding, and INTACT Predictor.  Training uses batch
size 256, learning rate $5\times10^{-4}$, weight decay $10^{-3}$, SIGReg 0.03,
and five epochs.  Encoder and Forward attention are forced to deterministic
Math-SDPA and every checkpoint stores backend and data fingerprints.

The controlled factorial contains forward-only LeWM, physical inverse-only,
intent-only, and Full INTACT.  Intent-only and Full are each trained with the
waypoint and goal-displacement coordinates, producing six cells.  All cells
use training seeds $\{0,42,3072\}$ and are evaluated at E1--E5.  At each epoch
we run native Direct inference where available, actor-disabled pure CEM
$300\times30$, Guarded A for actor-equipped cells, frozen probes, and latent
diagnostics.  The complete coverage contract and E1/epoch tables are reported
in \cref{app:extended_results}.

For mechanism analysis, predicted--expert linear CKA, local kNN overlap,
action $R^2$, and NLL are computed within each task's native action coordinates
and averaged equally over tasks; heterogeneous action spaces are never
zero-padded into one embedding.  The primary cohort contains both Full INTACT
interfaces, three seeds, and E1--E5.  A separate task-conditioned gauge audit
fits latent coordinate maps on paired, action-free calibration episodes and
tests correct pairing, shuffling, calibration size, cross-seed ceilings, and
reverse actor swaps under CLEAR Moderate.  Full estimator definitions and
controls appear in \cref{app:diagnostics,app:extended_results}.

\section{Results}
\label{sec:results}

\subsection{Single-Task Results}
\label{sec:single_results}

\paragraph{One epoch makes search optional.}
\Cref{tab:published_landscape} places the headline on a resource-aware axis.
Published rows retain their own protocols and are not paired significance
controls.  Every row in this single-task comparison trains a separate visual
model per domain; the shared-encoder study begins in
\cref{sec:multitask_results}.

\begin{table*}[t]
\centering
\scriptsize
\setlength{\tabcolsep}{3.1pt}
\renewcommand{\arraystretch}{0.98}
\caption{\textbf{Published landscape and matched task-specific INTACT inference matrix (official SR, \%).}
Published values are transcribed from the cited papers and are not paired
controls; $\dagger$ marks partial coverage or a distinct reproduction
protocol.  $\ddagger$ marks our matched-E1 reproduction: one task-specific
training seed with mean $\pm$ sample standard deviation over evaluation
manifests $\{0,1,42\}$ (100 episodes each), rather than training-seed
uncertainty.  INTACT values are mean $\pm$ sample standard deviation over three
training seeds, each averaged over three 100-episode evaluation seeds.  The
INTACT block audits all eight exact objective--inference pairings on their
matched checkpoint families.}
\label{tab:published_landscape}
\label{tab:single_direct}
\begin{adjustbox}{max width=\textwidth}
\begin{tabular}{lllcrrrrrr}
\toprule
Method / objective & Training & Deployment & Seq. & PushT & Cube & Reacher & TwoRoom & Macro \\
\midrule
DINO-WM \citep{zhou2025dinowm} & frozen $E$ & CEM (published) & 9k & $74.0\!\pm\!4.5$ & $86.0\!\pm\!4.7$ & $79.0\!\pm\!5.1$ & $100.0\!\pm\!.0$ & 84.75 \\
LeWM \citep{maes2026lewm} & E2E 10 ep & CEM $300{\times}(30/10)$ & 9k/3k & $96.0\!\pm\!4.0$ & $74.0\!\pm\!3.0$ & $86.0\!\pm\!5.0$ & $87.0\!\pm\!2.5$ & 85.75 \\
Fast-LeWM \citep{gao2026fastlewm} & E2E 10 ep & matched CEM & 9k & 96.0 & 80.0 & 88.0 & 98.0 & 90.50 \\
Fast-LeWM + SC \citep{gao2026fastlewm} & E2E 10 ep & CEM + score & 9k & 98.0 & 82.0 & 90.0 & 98.0 & 92.00 \\
Qantara$^\dagger$ \citep{rakhimov2026qantara} & E2E 10 ep & CEM $300{\times}30$, $K=4$ & 9k & $90.1\!\pm\!1.1$ & $93.7\!\pm\!.7$ & $80.9\!\pm\!1.8$ & $100.0\!\pm\!.0$ & 91.18 \\
GC-IDM \citep{nguyen2026gcidm} & LeWM + 50-ep head & Direct & 0 & $84.7\!\pm\!5.0$ & $99.3\!\pm\!1.2$ & $100.0\!\pm\!.0$ & $100.0\!\pm\!.0$ & 96.00 \\
PRISM$^\dagger$ \citep{wang2026prism} & LeWM + 50-ep head & MPPI $128{\times}30$ & 3,840 & $89\!\pm\!4$ & $79\!\pm\!6$ & -- & -- & -- \\
C-JEPA$^\dagger$ \citep{nam2026cjepa} & frozen $E$ + 30 ep & CEM $300{\times}30$ & 9k & 88.67 & -- & -- & -- & -- \\
\midrule
\multicolumn{9}{@{}l}{\emph{Matched Math-SDPA E1 reproduction on a frozen task-specific E1 LeWM encoder}} \\
GC-IDM (matched E1)$^\ddagger$ & frozen E1 $E$ + 1-ep head & Direct & 0 & $9.33\!\pm\!2.31$ & $69.33\!\pm\!7.02$ & $84.33\!\pm\!2.08$ & $99.33\!\pm\!.58$ & 65.58 \\
PRISM (matched E1)$^\ddagger$ & frozen E1 $E$ + 1-ep head & PoG-MPPI $128{\times}30$ & 3,840 & $18.67\!\pm\!5.51$ & $86.67\!\pm\!1.15$ & $66.33\!\pm\!1.15$ & $88.33\!\pm\!1.53$ & 65.00 \\
\midrule
\multicolumn{9}{@{}l}{\emph{Matched task-specific INTACT, E2E one epoch}} \\
\multirow{4}{*}{Waypoint INTACT}
& \multirow{4}{*}{E2E 1 ep} & Direct & 0 & $77.67\!\pm\!.88$ & $99.89\!\pm\!.19$ & $88.11\!\pm\!.69$ & $98.00\!\pm\!.33$ & $90.92\!\pm\!.30$ \\
& & Pure CEM $300{\times}30$ & 9k & $89.67\!\pm\!1.45$ & $67.67\!\pm\!1.33$ & $82.67\!\pm\!1.20$ & $77.33\!\pm\!2.03$ & $79.33\!\pm\!.75$ \\
& & Actor-on CEM $300{\times}30$ & 9k & $92.89\!\pm\!.84$ & $96.56\!\pm\!1.07$ & $80.67\!\pm\!2.08$ & $98.00\!\pm\!.67$ & $92.03\!\pm\!1.06$ \\
& & \best{Guarded A $128{\times}3$} & \best{384} & \best{$89.89\!\pm\!.38$} & \best{$99.78\!\pm\!.19$} & \best{$88.56\!\pm\!1.26$} & \best{$97.89\!\pm\!.19$} & \best{$94.03\!\pm\!.25$} \\
\addlinespace[1pt]
\multirow{4}{*}{\best{INTACT, goal displacement}}
& \multirow{4}{*}{\best{E2E 1 ep}} & \best{Direct} & \best{0} & $85.78\!\pm\!1.54$ & \best{$100.00\!\pm\!.00$} & \best{$97.67\!\pm\!.00$} & $97.89\!\pm\!1.26$ & \best{$95.33\!\pm\!.58$} \\
& & Pure CEM $300{\times}30$ & 9k & $88.44\!\pm\!1.17$ & $68.44\!\pm\!.77$ & $83.67\!\pm\!.67$ & $82.89\!\pm\!.84$ & $80.86\!\pm\!.51$ \\
& & Actor-on CEM $300{\times}30$ & 9k & \best{$93.56\!\pm\!.96$} & $96.89\!\pm\!.19$ & $86.67\!\pm\!.88$ & \best{$98.00\!\pm\!1.15$} & $93.78\!\pm\!.77$ \\
& & \best{Guarded A $128{\times}3$} & \best{384} & \best{$92.22\!\pm\!.69$} & \best{$99.78\!\pm\!.19$} & \best{$97.44\!\pm\!.77$} & \best{$98.00\!\pm\!1.15$} & \best{$96.86\!\pm\!.38$} \\
\bottomrule
\end{tabular}
\end{adjustbox}
\end{table*}

\best{INTACT uses the fewest reported full-data passes among methods with
disclosed epoch schedules: one end-to-end epoch, no frozen-encoder phase, and
no separately trained phase-2 controller.  Direct reaches 95.33 macro SR with
zero candidates and 2.9--5.5\,ms planner-side latency.  Guarded A evaluates
384 rather than 9,000 sequences and reaches 96.86 macro, placing INTACT at the
state-of-the-art level among currently reported JEPA controllers.}

The matched-E1 Math-SDPA rows ask a narrower question than the published headlines:
how much control can a one-epoch head amortize from the same frozen,
task-specific E1 LeWM representation?  GC-IDM uses zero-search Direct control,
whereas PRISM uses PoG-MPPI with $128{\times}30=3{,}840$ candidates.  Their
macros are 65.58 and 65.00, respectively, but the taskwise ordering reverses:
PRISM is higher on PushT/Cube, while GC-IDM is higher on Reacher/TwoRoom.
Because the objectives and inference budgets differ and each row uses one
training seed, this is a controlled E1 stress test rather than a paired
compute or significance ranking.  Descriptively, the end-to-end one-epoch
INTACT Direct row reaches 95.33 macro without candidate search.

The complete matched block in \cref{tab:single_direct} separates training
coordinates from deployment algorithms.  Goal-displacement INTACT reaches
$85.78/100.00/97.67/97.89$\% Direct SR on PushT/Cube/Reacher/TwoRoom, averaged
over three independently trained models per task.  Relative to published LeWM,
it uses one tenth as many full-data passes and removes up to 9,000 candidate
sequences.  Goal displacement outperforms the waypoint coordinate under every
matched interface: by 4.41/1.53/1.75/2.83 macro points for Direct, pure CEM,
actor-on CEM, and Guarded A, respectively.  The actor-disabled macros remain
close (80.86 versus 79.33), showing that broad raw-action search largely
discards the learned intent interface.  More decisively, Goal actor-on CEM
spends 9,000 sequences yet falls 1.55 points below zero-search Direct
(93.78 versus 95.33), whereas Guarded A raises Direct by 1.53 points to 96.86
with only 384 sequences.  The waypoint coordinate shows the same efficiency
ordering: Guarded A reaches 94.03 versus 92.03 for actor-on CEM.  Search is
therefore useful as local verification around the learned plan, while
unconstrained iterative search can recreate a train--deployment gap and waste
the learned conditional.

The separate CLEAR-LeWM Moderate audit preserves the same ordering: Direct and
Guarded A reach $83.33\!\pm\!1.00$ and $84.83\!\pm\!1.32$ macro, versus
$65.47\!\pm\!.77$ for pure CEM.  Reacher drops from 97.67 official to 49.56
Moderate after angle-topology correction, illustrating why CLEAR is reported
separately rather than pooled.  The complete taskwise table and protocol
contract are in \cref{tab:clear_main,app:clear_scope}.

The matched actor-sharing audit is deferred to
\cref{tab:actor_sharing,app:actor_sharing}: separating the two calls loses 5.67
PushT points at matched capacity, and even doubling actor capacity remains 4.00
points below the fully shared conditional.

\begin{figure*}[t]
\centering
\includegraphics[width=0.92\textwidth]{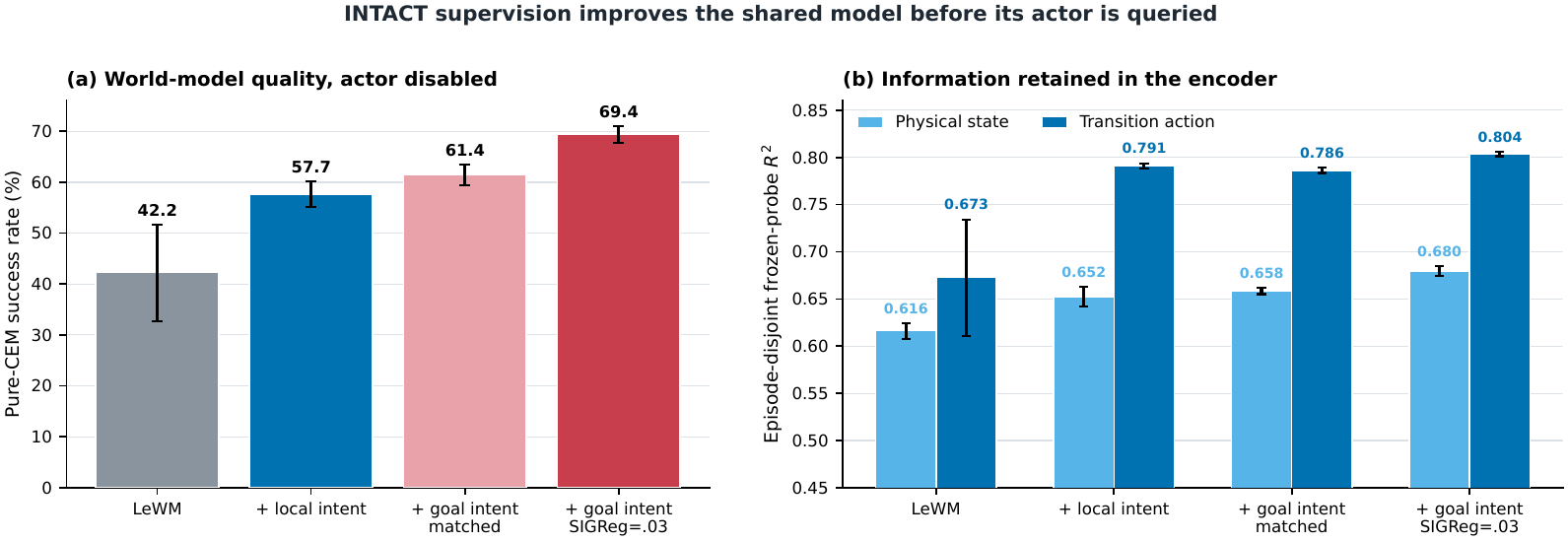}
\caption{\textbf{Single-task action likelihoods shape the representation.}
Actor-disabled CEM evaluates the encoder--Forward stack under a fixed
optimizer; episode-disjoint probes measure readable state and transition
action.  Physical inverse supplies attached successor shaping, whereas the
goal branch updates current-state coordinates and the shared actor with a
stop-gradient goal anchor.  Error bars are sample standard deviations over
three PushT training seeds.}
\label{fig:query_representation}
\end{figure*}

\paragraph{The encoder improves before the actor is used.}
On PushT, adding physical inverse supervision raises actor-disabled CEM
$30\times10$ from $42.22\!\pm\!9.50$ to $57.67\!\pm\!2.52$\% and frozen
endpoint-action $R^2$ from $.673\!\pm\!.062$ to $.791\!\pm\!.003$.  Adding
matched waypoint supervision raises CEM to $61.44\!\pm\!2.04$\%; the larger
Direct improvement primarily identifies the deployable conditional.  The
low-SIGReg waypoint model reaches $69.44\!\pm\!1.64$\% actor-off CEM and
$77.67\!\pm\!.88$\% Direct.  Its lower rank is not collapse: per-dimension
standard deviation remains near one and no dimension collapses.

\paragraph{The taskwise gains isolate the learned interface.}
Replacing the waypoint coordinate by goal displacement changes Direct SR by
$+8.11/+0.11/+9.56/-0.11$ points on PushT/Cube/Reacher/TwoRoom.  The gain is
therefore concentrated in contact-sensitive PushT and continuously actuated
Reacher, while Cube and TwoRoom are already near saturation.  Under
actor-disabled pure CEM, the same replacement improves macro SR by only 1.53
points (79.33 to 80.86); after the learned actor is restored, the Direct gain is
4.41 points and Guarded A reaches 96.86.  This separation matters: the
objective improves the world representation, but its larger benefit comes from
exposing a deployment condition that the shared actor can interpret directly.

Coordinate grammar, SIGReg, straightening, Huber, extra-epoch, interaction,
and planner-radius controls are reported together in
\cref{app:design_selection,app:extended_results}.  They select the final model
but are not repeated as headline experiments.

\FloatBarrier

\subsection{Multi-Task Results}
\label{sec:multitask_results}

We test the linked representation, control, and stability consequences on
matched checkpoints below.

\paragraph{One shared encoder improves all four domains.}
At E5, goal-displacement INTACT reaches
$89.39\!\pm\!.77$ macro Direct SR, versus $66.17\!\pm\!2.67$ for matched
shared-encoder LeWM with CEM $300\times30$.  It improves LeWM by
5.66/32.23/12.56/42.44 points on PushT/Cube/Reacher/TwoRoom and exceeds the
85.75 published task-specific LeWM macro.  Against the matched intent-only
cell, adding the attached physical likelihood raises macro SR by
$8.78\!\pm\!1.63$ points; the gain is positive on three tasks and $-0.44$ on
the saturated Cube task.  This is the cleanest evidence that Full INTACT is not
merely a goal-conditioned behavior-cloning head.

\begin{figure*}[p]
\centering
\includegraphics[width=0.95\textwidth]{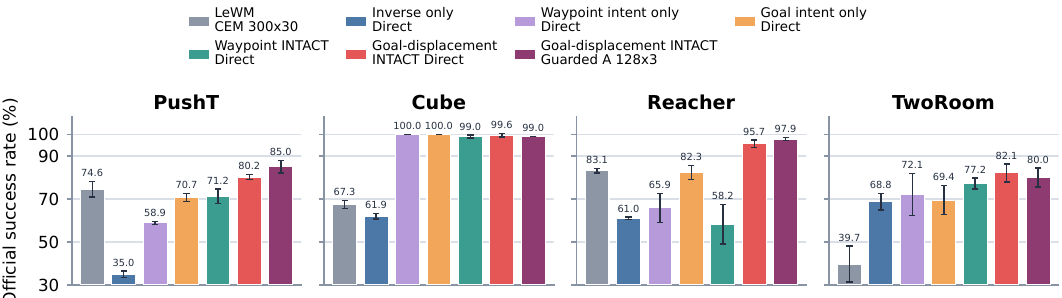}
\caption{\textbf{Controlled E5 success across four tasks.} Bars visualize the
native rows of \cref{tab:multitask_e5}: LeWM uses CEM $300{\times}30$, learned
heads use Direct, and the final bar uses Guarded A $128{\times}3$. Values are
means $\pm$ sample standard deviations over three training seeds. The ordinate
starts at 30\% to expose high-SR differences and is not a zero-based effect size.}
\label{fig:multitask_e5_sr_summary}
\vspace{5pt}
\begin{minipage}{0.98\textwidth}
\centering
\scriptsize
\setlength{\tabcolsep}{4.5pt}
\renewcommand{\arraystretch}{1.20}
\captionof{table}{\textbf{Controlled Math-SDPA shared-encoder E5 result.} Native
interfaces, actor-disabled pure CEM, and optional verification are evaluated
on the same checkpoints. Values are mean $\pm$ sample standard deviation over
three training seeds; each seed averages official evaluation seeds
$\{0,1,42\}$ with 100 episodes each.}
\label{tab:multitask_e5}
\begin{adjustbox}{max width=\textwidth}
\begin{tabular}{llrrrrr}
\toprule
Training cell & Evaluation & PushT & Cube & Reacher & TwoRoom & Macro \\
\midrule
\multicolumn{7}{@{}l}{\emph{Native interface (Direct where available)}} \\
LeWM & CEM $300{\times}30$ & $74.56\!\pm\!3.67$ & $67.33\!\pm\!1.86$ & $83.11\!\pm\!.96$ & $39.67\!\pm\!8.39$ & $66.17\!\pm\!2.67$ \\
Inverse only & Direct & $35.00\!\pm\!1.45$ & $61.89\!\pm\!1.26$ & $61.00\!\pm\!.58$ & $68.78\!\pm\!3.83$ & $56.67\!\pm\!.22$ \\
Waypoint intent only & Direct & $58.89\!\pm\!.69$ & \best{$100.00\!\pm\!.00$} & $65.89\!\pm\!6.71$ & $72.11\!\pm\!9.83$ & $74.22\!\pm\!3.23$ \\
Goal intent only & Direct & $70.67\!\pm\!1.76$ & \best{$100.00\!\pm\!.00$} & $82.33\!\pm\!3.18$ & $69.44\!\pm\!6.75$ & $80.61\!\pm\!2.11$ \\
Waypoint INTACT & Direct & $71.22\!\pm\!3.36$ & $99.00\!\pm\!.67$ & $58.22\!\pm\!9.10$ & $77.22\!\pm\!2.52$ & $76.42\!\pm\!2.32$ \\
\best{Goal-displacement INTACT} & \best{Direct} & \best{$80.22\!\pm\!1.26$} & \best{$99.56\!\pm\!.77$} & \best{$95.67\!\pm\!1.76$} & \best{$82.11\!\pm\!4.11$} & \best{$89.39\!\pm\!.77$} \\
\midrule
\multicolumn{7}{@{}l}{\emph{Actor-disabled pure CEM $300{\times}30$}} \\
LeWM & Pure CEM & $74.56\!\pm\!3.67$ & $67.33\!\pm\!1.86$ & $83.11\!\pm\!.96$ & $39.67\!\pm\!8.39$ & $66.17\!\pm\!2.67$ \\
Inverse only & Pure CEM & $75.78\!\pm\!1.02$ & $66.78\!\pm\!2.36$ & $81.78\!\pm\!1.64$ & $51.44\!\pm\!1.07$ & $68.94\!\pm\!.43$ \\
Waypoint intent only & Pure CEM & $78.44\!\pm\!4.72$ & $68.67\!\pm\!1.76$ & $83.22\!\pm\!2.69$ & $39.11\!\pm\!8.17$ & $67.36\!\pm\!4.06$ \\
Goal intent only & Pure CEM & \best{$80.78\!\pm\!2.36$} & $68.00\!\pm\!3.18$ & \best{$84.67\!\pm\!1.33$} & $40.11\!\pm\!8.30$ & $68.39\!\pm\!2.22$ \\
Waypoint INTACT & Pure CEM & $78.22\!\pm\!3.02$ & $69.56\!\pm\!2.80$ & $83.00\!\pm\!2.19$ & $46.44\!\pm\!3.34$ & $69.31\!\pm\!.68$ \\
\best{Goal-displacement INTACT} & \best{Pure CEM} & $78.11\!\pm\!1.17$ & \best{$70.33\!\pm\!3.00$} & $80.33\!\pm\!2.03$ & \best{$51.56\!\pm\!4.53$} & \best{$70.08\!\pm\!1.13$} \\
\midrule
\multicolumn{7}{@{}l}{\emph{Optional local verification}} \\
\best{Goal-displacement INTACT} & \best{Guarded A $128{\times}3$} & \best{$85.00\!\pm\!2.85$} & \best{$99.00\!\pm\!.00$} & \best{$97.89\!\pm\!.69$} & \best{$80.00\!\pm\!4.41$} & \best{$90.47\!\pm\!.84$} \\
\bottomrule
\end{tabular}
\end{adjustbox}
\par\vspace{5pt}\footnotesize\raggedright
\noindent\textbf{Conclusion.}
Preserving and organizing the state-conditioned intent--action family relation
jointly determines whether actions can be read out directly, whether the
actor-disabled world model remains plannable, whether low-budget search is
effective, and whether shared multi-task training is stable.
\end{minipage}
\vspace{9pt}

\begin{minipage}[t]{0.485\textwidth}
\normalsize\raggedright
\noindent\textbf{The preferred coordinate changes during shared training.}
At E1, Waypoint INTACT is the strongest first-pass cell at
$42.69\!\pm\!7.19$ macro, compared with $36.83\!\pm\!5.63$ for Goal INTACT;
its advantage is concentrated in PushT, Cube, and high-variance TwoRoom. By E5,
the ordering reverses: goal displacement reaches $89.39\!\pm\!.77$ versus
$76.42\!\pm\!2.32$ for waypoint, with per-task differences of
$+9.00/+0.56/+37.45/+4.89$ points. Thus E1 measures optimization speed, whereas
E5 reveals the stronger converged intent coordinate.
\end{minipage}\hfill
\begin{minipage}[t]{0.485\textwidth}
\normalsize\raggedright
\noindent\textbf{Actor-disabled controls separate representation from readout.}
Goal INTACT reaches $70.08\!\pm\!1.13$ pure-CEM macro versus
$66.17\!\pm\!2.67$ for LeWM. Its native Guarded A result is
$90.47\!\pm\!.84$, only $1.08\!\pm\!.14$ above Direct, again supporting search
as a verifier. E1 matrices, all inference modes, CLEAR Moderate scores, and
controlled E1/E5 probes are in
\cref{tab:multitask_core,tab:multitask_inference_decomp,tab:multitask_diagnostics}.
The higher effective rank of E5 waypoint intent-only (93.87 versus 89.26) but
lower SR (74.22 versus 89.39) confirms that rank is not a semantic certificate.
\end{minipage}
\end{figure*}
\clearpage

\begin{table*}[p]
\centering
\small
\setlength{\tabcolsep}{4.5pt}
\renewcommand{\arraystretch}{1.45}
\caption{\textbf{Theory-linked evidence for shared four-task INTACT.}
Correlations use official Direct SR; the actor-disabled and gauge rows are
matched interventions that test representation and correspondence directly.}
\label{tab:multitask_theory_evidence}
\begin{tabularx}{\textwidth}{@{}>{\raggedright\arraybackslash}p{3.0cm}>{\raggedright\arraybackslash}p{7.15cm}>{\raggedright\arraybackslash}X@{}}
\toprule
Evidence & Measured result & Interpretation \\
\midrule
Pred.--expert kNN\newline (45 eligible E1--E5 checkpoints)
& \makecell[l]{Pooled \best{$r=.954$} $[.928,.969]$;\\
adjusted \best{$r=.902$}.\\
Within Waypoint/Goal: \best{$.968/.981$}.}
& Local action-family neighborhoods track deployable control. \\
Pred.--expert linear CKA\newline (same 45 checkpoints)
& \makecell[l]{Pooled \best{$r=.897$} $[.837,.930]$.\\
Within Waypoint/Goal: \best{$.979/.986$}.\\
Leave-one-epoch-out: \best{$[.888,.914]$}.}
& Global family geometry remains predictive across interfaces and epochs. \\
Pointwise action $R^2$\newline (same 45 checkpoints)
& $r=.815$.
& Recovering one expert action is weaker than preserving the family relation. \\
Effective-rank inversion\newline (controlled E5 cells)
& Rank $93.87>89.26$, while SR $74.22<89.39$.
& Latent spread is a capacity check, not a semantic certificate. \\
Actor-disabled planning\newline (matched E5 checkpoints)
& Pure-CEM macro $66.17\rightarrow\mathbf{70.08}$.
& Action losses improve the Encoder--Forward stack before actor execution. \\
Paired gauge intervention\newline (21,600 CLEAR Moderate episodes)
& Shuffled $9.46\%\rightarrow\mathbf{68.04\%}$ paired ($+58.58$ pp).
& Correct task-local correspondence, rather than coordinate identity, restores control. \\
\bottomrule
\end{tabularx}
\vspace{14pt}

\begin{minipage}[t]{0.485\textwidth}
\normalsize\raggedright
\noindent\textbf{Matched interfaces separate representation, readout, and search.}
With the actor disabled, Goal INTACT changes LeWM by
$+3.55/+3.00/-2.78/+11.89$ points on PushT/Cube/Reacher/TwoRoom: representation
shaping is positive overall but not uniformly sufficient.  Restoring the
native Direct interface raises macro SR from 70.08 to 89.39, and adding Guarded
A contributes only another $1.08\!\pm\!.14$ points.  Conversely, Full Goal
INTACT exceeds goal-intent-only Direct by 9.55/$-0.44$/13.34/12.67 points.  The
three comparisons jointly rule out a single-factor explanation: physical
likelihood improves the Encoder--Forward stack, shared intent readout supplies
most closed-loop control, and low-budget search performs a final local check.

\medskip
\noindent\textbf{Summary.}
Physical inversion improves action-recoverable world features; deployment
intent identifies a zero-search action interface; sharing forces both condition
families to be interpreted by one conditional law.
\end{minipage}\hfill
\begin{minipage}[t]{0.485\textwidth}
\normalsize\raggedright
Goal displacement is the
strongest final coordinate without forcing latent trajectories to be linear.
The single- and multi-task results agree: Direct is the native interface,
actor-disabled planning remains useful, and small search is best understood as
optional verification.

\medskip
\noindent\textbf{Why these measurements follow the theory.}
The quotient-to-law view in \cref{prop:quotient_image,prop:proper} predicts that
control should follow agreement between the predicted and expert action-law
families. kNN and CKA operationalize its local neighborhoods and global centered
geometry; pointwise $R^2$ deliberately asks the weaker single-action question.
Effective rank then serves as a capacity falsifier, actor-disabled planning
tests representation before readout, and the paired gauge intervention tests
correspondence while allowing coordinates to change. Their observed ordering
matches these roles: family metrics are strongest, higher rank can accompany
lower SR, physical supervision improves pure CEM, and correct pairing restores
58.58 points without requiring coordinate identity.
\end{minipage}
\end{table*}
\clearpage

\begin{figure*}[p]
\centering
\includegraphics[width=0.96\textwidth]{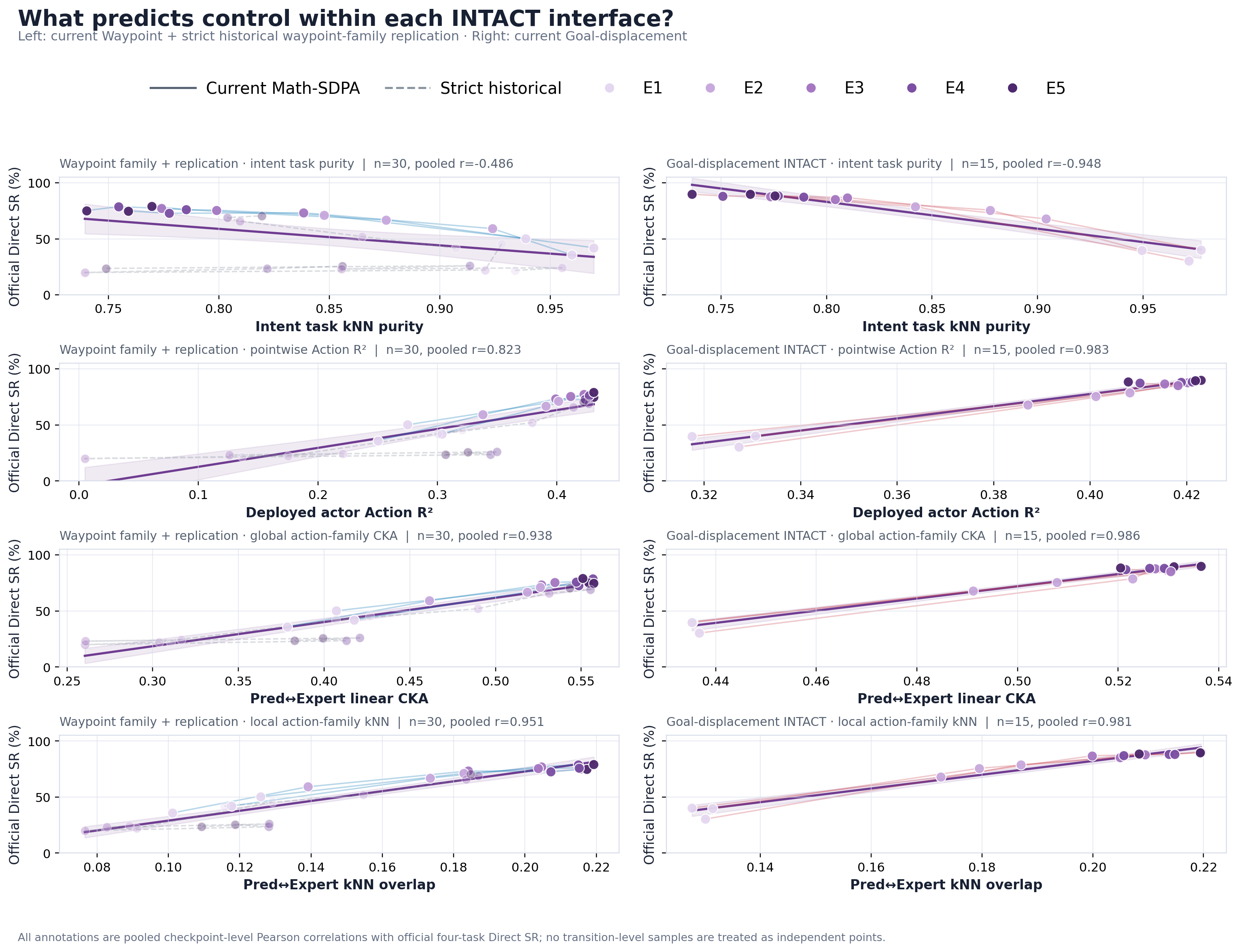}
\caption{\textbf{Intent--action relation, rather than task clustering, tracks
control.} Columns separate waypoint and goal-displacement interfaces; rows
compare task purity, pointwise action $R^2$, predicted--expert CKA, and local
kNN overlap.  Each point is one checkpoint and each annotation is a pooled
checkpoint-level Pearson correlation with official four-task Direct SR.
Historical runs are lower opacity and never rank the controlled objectives.}
\label{fig:intent_action_alignment}
\vspace{3pt}

\begin{minipage}[t]{0.485\textwidth}
\small\raggedright
\noindent\textbf{Family conversion predicts closed-loop control.}
Across the 30 primary Math-SDPA Full-INTACT checkpoints, predicted--expert CKA
and local kNN overlap correlate with official Direct SR at $r=.861$ and $.926$.
Within complete waypoint/goal learning trajectories, CKA reaches
$.979/.986$ and kNN $.968/.981$.  Pooling the separately labeled strict
replication cohort gives CKA $r=.897$ (95\% run-cluster bootstrap CI
$[.837,.930]$) and kNN $r=.954$ ($[.928,.969]$); adjusted kNN remains $.902$.
By contrast, intent task purity is negatively associated with SR.  Cleaner
task clusters are therefore not the desired relation: what matters is whether
the state-conditioned intent family is carried into the corresponding action
family.
\end{minipage}\hfill
\begin{minipage}[t]{0.485\textwidth}
\small\raggedright

\noindent\textbf{Gauge intervention separates correspondence from actor function.}
Correctly paired task-local calibration recovers 68.04\% CLEAR Moderate SR,
whereas identity and shuffled-pair maps obtain 7.96\% and 9.46\%.  Same-
objective swaps define a high alignment ceiling, and LeWM/inverse backbones can
also align into a trained Full actor.  The reverse direction is decisive: a
Full backbone mapped into a LeWM actor remains at 5.08\%, while an inverse actor
recovers only inverse-level functionality.  A coordinate map can restore a
learned conditional but cannot create one absent from the target actor.  The
72-cell, 21,600-episode audit and calibration curves are in
\cref{fig:task_conditioned_gauge,app:extended_results}.
\end{minipage}
\par\vspace{6pt}
\begin{minipage}{0.98\textwidth}
\small\raggedright
\noindent\textbf{The eight panels reject a clustering-only explanation.}
Across both deployment coordinates, task purity has the wrong sign, pointwise
action recovery is informative but weaker, and relation-aware CKA/kNN give the
most consistent monotone trends. The agreement between current Math-SDPA runs
and the separately marked historical cohort further shows that the result is
not produced by one checkpoint family. Together with the gauge intervention,
the figure supports a narrower claim than global latent isomorphism: successful
control requires each state-conditioned intent neighborhood to retain the
action-family correspondence that its task-specific head can read.
\end{minipage}
\end{figure*}

\clearpage

\section{Limitations}
\label{sec:limitations}

\paragraph{Three seeds remain a coarse scaling estimate.}
The controlled multi-task matrix is complete for all six objective cells,
three training seeds, three evaluation seeds per checkpoint, and E1--E5 under
Math-SDPA. Official Direct, pure CEM, Guarded A, and CLEAR Moderate have full
applicable coverage. The primary conversion audit is likewise complete for
both Full INTACT interfaces, three seeds, and E1--E5. The separately marked
historical replication cohort predates the backend audit and is used only for
mechanism robustness, not objective ranking. The final single-task
goal-displacement study now uses
three training seeds on all four tasks, but three seeds still provide only a
coarse estimate of training variability.

\paragraph{The action quotient is identified only on demonstrated support.}
\Cref{prop:proper} applies to endpoint conditions with positive sampling
probability. It does not identify arbitrary counterfactual goals, nor does a
single expert action reveal whether another action would have been equally
valid. The full goal displacement preserves direction and distance but remains
an approximate request around obstacles, contact switches, and multimodal
demonstrations. The local-intent call physically anchors the
representation, while the goal-intent call relies on demonstration support;
neither construction proves correct behavior outside that support.
Using the same goal-intent construction and INTACT Predictor at training and deployment removes
the forward-model/action-proposal interface mismatch, not every train--test
gap: after the first direct step, autoregressive control conditions on predicted
latents whose distribution can drift from encoded demonstration states.

\paragraph{The current diagnostics are necessary checks, not certificates.}
Effective rank and SRS can reward isotropic noise. $G_{\mathrm{sep}}$ measures
response to deterministic goal changes, but two different goals can require
the same current action, so it is only a proxy for an action-distinct quotient.
Expert NLL shares the INTACT Predictor used by Direct and is therefore not an independent
causal variable. A stronger test would pair action-equivalent and
action-distinct counterfactual conditions, report normalized Jacobian spectra,
and predict held-out checkpoint failures without refitting thresholds. The
t-SNE geometry is qualitative and is not used for model selection.

\paragraph{Direct Gaussian control can hide multimodality.}
The deployed action is the mean of a diagonal Gaussian. At junctions or contact
transitions, this mean can lie between valid modes. Search can verify a proposal
against forward dynamics, but then loses the zero-candidate identity. Mixture
actors, sampling from the INTACT Predictor's learned conditional uncertainty instead of a
unit-variance raw-action distribution, uncertainty-triggered verification, and
quotient-aware mode selection are natural extensions; our current evidence
supports a simple cross-task default rather than a universal action
distribution.

\paragraph{Gauge equivalence is task-manifold conditioned.}
The base Full--goal-only swap uses two training seeds; the expanded 72-cell
audit strengthens causal controls but does not establish exact conjugacy or a
single global affine chart. Correct pairing is necessary, and same-objective
cross-seed swaps define a high alignment ceiling. However, LeWM and
inverse-only backbones can also be aligned into a trained Full actor, so
alignability is not an INTACT-exclusive representation property. Conversely,
reverse swaps show that a gauge cannot create a deployment map absent from the
target actor. Cross-task and leave-one-task-out transfer still fail. This
task-local scope is sufficient for task-specific control heads; a single
universal actor would require explicit cross-task alignment, adapters, or
stronger joint coverage.

\paragraph{Scope and evaluation.}
Experiments use four simulated tasks, fixed image goals, task-specific action
heads, and offline expert trajectories. We have not established real-robot,
out-of-distribution goal, distractor, or cross-embodiment generalization.
The present headline uses the official LeWM benchmark to preserve direct
comparability. Our stricter legality-aware audit is reported separately and
should be validated independently. Finally, PRISM, GC-IDM, SMWM, QuoVLA, and related concurrent systems use
different data, training schedules, and evaluators. We therefore compare their
interfaces and assumptions, not unmatched success-rate numbers.

\FloatBarrier
\section{Conclusion}

INTACT starts from a mismatch that is easy to overlook: a physical
successor is a valid inverse-dynamics condition but is unavailable before
acting, whereas a deployable goal intent is available but is not itself an observed
transition. Their useful commonality is neither endpoint type nor Euclidean
proximity. At a fixed world state, it is the expert action law they induce. This
conditional action quotient leads to a deliberately small statistical unit:
each supervised conditional consists of one demonstrated state, one supported
endpoint condition, and one correct action. The forward JEPA and SIGReg remain responsible for preserving the
richer world representation. In the full objective, the two action losses also
remain distinct: local-intent likelihood shapes the encoder through the attached real
successor, whereas goal-intent likelihood shapes the current-state encoding toward
the deployable conditional and has no successor-gradient path. This is a
branch-local stop-gradient, not a detached batch or encoder: the same future
latent is plastic when used as a physical successor and fixed when reused as a
deployment anchor.
Taken separately, these are familiar objectives: local inverse dynamics and a
GCSL-like hindsight action likelihood. INTACT begins at their irreducible
coupling: one weight-shared, graph-isomorphic conditional operator must explain
both families in action-law space, while asymmetric endpoint gradients preserve
their distinct physical and deployment roles and the Forward Predictor retains
richer world information. The method therefore does not reduce to placing two
losses beside each other.

The controlled evidence supports the resulting predictions. In single-task
training, the final goal-displacement model reaches 85.78\%, 100.00\%, 97.67\%,
and 97.89\% official Direct SR after one epoch with no candidate search, with
three training seeds on every task.  The completed $2\times4$ task-specific
matrix further shows that goal displacement exceeds the waypoint coordinate
under Direct, pure CEM, actor-on CEM, and Guarded A.  On the final coordinate,
actor-on CEM spends 9,000 sequences yet falls from 95.33\% Direct macro to
93.78\%, whereas plan-centered Guarded A reaches 96.86\% with 384 sequences.
Controlled head-grammar interventions
show a 7.00-point A$\rightarrow$E gain along the matched-product path. The new
G control improves A by only 0.89 points when centering the product alone,
whereas matching the centered first-order intent with its product yields a
3.89-point residual--product interaction; under the final goal coordinate,
removing that matched product costs 2.56 points. Earlier matched PushT attribution shows that the waypoint
encoder reaches $69.44\!\pm\!1.64$\% with the actor removed under pure CEM
$30{\times}10$, separating representation shaping from proposal quality. In
shared four-task training, the deterministic Math-SDPA factorial gives Full
INTACT the strongest E1 Direct macro, $42.69\!\pm\!7.19$\%, and lower TwoRoom
variation than waypoint-intent-only. At E5, Goal-displacement INTACT reaches
$89.39\!\pm\!0.77$\% Direct macro and exceeds matched three-seed
goal-intent-only by $8.78\!\pm\!1.63$ points. Actor-disabled pure CEM improves
from $66.17\!\pm\!2.67$\% for LeWM to $70.08\!\pm\!1.13$\%, while Guarded A
raises the Full model to $90.47\!\pm\!0.84$\%; search therefore remains a
compatible verifier rather than the learned interface. Across 45 eligible simultaneous-dual-likelihood
E1--E5 checkpoints, predicted--expert CKA/kNN correlate with official Direct SR
at $r=0.897/0.954$; kNN remains $r=0.902$ after epoch/cohort adjustment. This
mechanism audit complements, rather than replaces, the controlled performance
comparison.
On matched PushT waypoint training, parameter-matched independent actors lose
5.67 Direct points, and even double capacity remains 4.00 points below the
fully shared operator. This directly supports sharing as a semantic coupling,
not a parameter-count accident. The A--G head grammar supplies a complementary
intra-call constraint: the first-order intent and its state--intent bilinear
interaction must carry matched semantics inside each call. The deployment loss
is still a goal-conditioned imitation likelihood; these controls show why the
complete method is not reducible to goal-only or post-hoc imitation: it couples
the GCSL-like deployment likelihood end to end with JEPA dynamics and the same
shared conditional operator on attached one-chunk physical successors.

The 21,600-episode gauge audit supplies a direct causal check. Correct pairing
recovers 68.04\% CLEAR Moderate SR, versus 9.46\% after pair shuffling, and the
calibration curve saturates near 64/full episodes. Same-objective swaps set a
high alignment ceiling, while LeWM/inverse backbones also align into a trained
Full actor; alignability is therefore not unique to INTACT. In the reverse
direction, Full-backbone-to-LeWM-actor remains at 5.08\%, and an inverse actor
recovers only inverse-level functionality. A coordinate gauge can restore a
learned map but cannot invent a deployment conditional absent from the actor.

The complete 576-job planner audit adds a compatible efficiency result.
Guarded A centers a $128{\times}3$ raw-action residual search on the
deterministic INTACT plan and reaches 96.86\% macro and 92.22\% worst-task SR,
exceeding matched pure CEM $300{\times}30$ by 16.00 macro points with
23.44$\times$ fewer sampled candidate sequences. Ordinary
$\sigma_0=0.25$ is only 0.08 macro point lower. A diagnostic L2-trust variant
raises PushT to 92.78\% but is less uniform across tasks, so it is not part of
the recommended inference interface. The finite correction region improves
contact-sensitive PushT without discarding the precision already present in
Reacher, whereas iterative actor-particle sampling compounds mixture and model
error.

The broader lesson is that amortized control need not be attached after world
model learning or reduced to a planner-side sampling trick. A single JEPA can
retain predictive world state, learn which conditions are equivalent for the
current action, and expose that conditional directly at deployment. Search then
becomes an optional verifier rather than the mandatory interface to the model.
At a more general level, the construction is not tied to JEPA: whenever two
condition families share supervised behavioral labels, a shared proper
conditional operator can align their quotient semantics while asymmetric
gradient routing preserves the role of each family. INTACT therefore offers a
template for relation-preserving family mappings beyond world models, without
requiring pointwise coordinate collapse.

\bibliographystyle{unsrtnat}
\bibliography{references}

\clearpage
\appendix
\section*{Appendix}
\addcontentsline{toc}{section}{Appendix}
\setcounter{figure}{0}
\renewcommand{\thefigure}{A\arabic{figure}}
\renewcommand{\theHfigure}{appendix.\arabic{figure}}
\setcounter{table}{0}
\renewcommand{\thetable}{A\arabic{table}}
\renewcommand{\theHtable}{appendix.\arabic{table}}
\section{Proofs and Statistical Interpretation}
\label{app:proofs}
\renewcommand{\qedsymbol}{}

\begin{proof}[Proof of \cref{prop:nonidentifiability}]
Let $S=\operatorname{supp}(z,z^+)$. The successor-only risk is an integral of
$-\log p_\eta(a\mid z,y)$ over triples with $(z,y)\in S$. If two actors agree
on $S$, their integrands and therefore their risks are equal almost everywhere.
Changing either actor on a goal intent $m$ outside $S$ cannot change that integral.
Thus the successor-only objective does not identify deployment behavior at
$m$.
\end{proof}

\begin{proof}[Proof of \cref{prop:proper}]
Condition on an endpoint pair $(z,y)$ with positive probability. Its expected
negative log likelihood decomposes as
\begin{equation}
\begin{aligned}
&\mathbb E_{a\sim p_E^\star(\cdot\mid z,y)}
  [-\log p_\eta(a\mid z,y)]\\[-1pt]
&\qquad=H[p_E^\star(\cdot\mid z,y)]
  +D_{\mathrm{KL}}(p_E^\star\|p_\eta).
\end{aligned}
\end{equation}
The entropy is constant in $\eta$ and the KL term is nonnegative, with equality
only when the modeled and expert conditional laws agree almost everywhere.
Applying this argument to every supported endpoint proves the result. If
$y\sim_z y'$, both recovered laws equal the same expert law; if the expert laws
differ, a proper likelihood preserves that distinction.
\end{proof}

For a Bernoulli endpoint sample, linearity of expectation gives
\cref{eq:endpoint_nll}. A paired estimator of
$(1-p)\ell(a;z^+)+p\ell(a;z+m^{\rm goal})$ can have the same population expectation. Our
claim is consequently not that sampling changes the population optimum by
itself. It keeps one conditional observation per transition, fixes the
action-loss scale, and changes finite-batch gradient covariance under adaptive
optimization. The paired controls test this finite-sample prediction.

\section{Implementation and Evaluation Identities}
\label{app:implementation}

The encoder is ViT-Tiny/14 followed by a two-layer projection to 192 dimensions.
The forward model is a six-layer causal transformer with 16 attention heads and
history length three. One model action contains five environment actions. The
shared INTACT Predictor is a three-layer width-1024 MLP with LayerNorm and GELU; it emits
a diagonal-Gaussian mean and log standard deviation clipped to $[-5,2]$.
Images use ImageNet normalization and actions use training-split statistics.

The backbone has 18.03M trainable parameters and the shared INTACT Predictor adds 3.11M.
Task heads are separate; the encoder and projector are shared across all four
tasks. Training calls \texttt{seed\_everything(seed, workers=True)} before
dataset and model construction. Checkpoints are loaded in strict mode and every
evaluation records the checkpoint hash, manifest hash, evaluator fingerprint,
task, protocol, solver identity, and seed.

\begin{table}[t]
\centering
\scriptsize
\caption{Inference identities enforced by the evaluator. A zero in Direct is
an assertion checked from the trace, not merely a requested setting.}
\label{tab:inference_identity}
\begin{adjustbox}{max width=\columnwidth}
\begin{tabular}{lccc}
\toprule
Mode & Actor & Candidates/iter. & Cost calls \\
\midrule
Direct & \cmark & 0 & 0 \\
Pure CEM $5{\times}2$ & \xmark & 5 per iteration & $>0$ \\
Pure CEM $30{\times}10$ & \xmark & 30 per iteration & $>0$ \\
Pure CEM $300{\times}30$ & \xmark & 300 per iteration & $>0$ \\
\bottomrule
\end{tabular}
\end{adjustbox}
\end{table}

\section{Extended Method Derivations and Control Pseudocode}
\label{app:method_derivations}

\subsection{Waypoint route and equal-progress assumption}
The waypoint attribution route solves, at teacher step $k$ with $r=H-k$,
\begin{equation}
\min_{\delta_k,\ldots,\delta_{H-1}}
\sum_{j=k}^{H-1}\|\delta_j\|_2^2
\quad\text{s.t.}\quad
\sum_{j=k}^{H-1}\delta_j=z_g-z_{t+k}.
\label{eq:query_variational}
\end{equation}
Strict convexity gives the equal-progress request
\begin{equation}
m_{t+k,H-k}^{\rm waypoint}
=\frac{\operatorname{sg}(z_g)-z_{t+k}}{H-k},
\label{eq:query}
\end{equation}
implemented in the endpoint grammar as
\begin{equation}
q_{t+k,H-k}=z_{t+k}+m_{t+k,H-k}^{\rm waypoint}.
\label{eq:waypoint_endpoint}
\end{equation}
It is paired with the demonstrated one-chunk action,
\begin{equation}
\mathcal L_{\rm goal}^{\rm waypoint}
=\mathbb E\!\left[-\log p_\eta\!\left(
a_{t+k}\,\middle|\,
\substack{z_{t+k},q_{t+k,H-k},\\a_{t+k-1}}
\right)\right].
\label{eq:query_nll}
\end{equation}
This route narrows train--deployment support but divides away remaining-distance
magnitude.  The final raw-goal route instead uses
$m_t^{\rm goal}=\operatorname{sg}(z_g)-z_t$ and lets the nonlinear actor choose
state-dependent progress.

The two routes share the same branch-local gradient rule.  The future endpoint
is detached only in the deployment occurrence; it remains attached when it is
the real successor of the last physical transition.  Neither a complete batch
nor the encoder is detached.

\subsection{Why intent coordinates do not straighten trajectories}
A realized transition under the goal-conditioned actor may be written
descriptively as
\begin{equation}
\Delta z_t^{\rm real}
=\alpha_t(z_t,m_t^{\rm goal},a_{t-1})m_t^{\rm goal}+r_t,
\label{eq:adaptive_progress}
\end{equation}
where $r_t$ can contain detours, contact changes, and multimodal residuals.
The Forward Predictor updates reachability and the goal residual,
\begin{align}
&a_t\sim\pi_\eta(\cdot\mid z_t,m_t^{\rm goal},a_{t-1}),\\
&\widehat z_{t+1}=F_\phi(\widehat z_t,a_t),\\
&m_{t+1}^{\rm goal}=z_g-\widehat z_{t+1}.
\label{eq:adaptive_rollout}
\end{align}
This does not impose the Euclidean trajectory penalty
\begin{equation}
\begin{aligned}
\mathcal L_{\rm lin}
&=\sum_{k=1}^{H-1}
\left\|z_{t+k}-[(1-\alpha_k)z_t+\alpha_kz_g]\right\|_2^2,\\
\alpha_k&=k/H,
\end{aligned}
\label{eq:direct_linearization}
\end{equation}
whose matched negative control appears in \cref{app:linearity_control}.

\subsection{Alternative endpoint sampling and end-to-end necessity}
One may draw $b\sim\operatorname{Bernoulli}(p)$ and use one endpoint per
transition,
\begin{align}
&y=(1-b)z^++b(z+m^{\rm goal}),\\
&\mathcal L_{\rm end}
=\mathbb E[-\log p_\eta(a\mid z,y,a_{t-1})].
\label{eq:endpoint_nll}
\end{align}
Its expectation is
$(1-p)\mathcal L_{z^+}+p\mathcal L_{\rm goal}$, although paired evaluation of
both calls changes finite-batch gradient covariance.  More importantly, if
$E(o)=E(o')$ for observations requiring different action laws under the same
intent, no frozen downstream actor can undo that aliasing.  Backpropagating the
proper likelihood through $E$ exposes the missing controllable distinction
during representation learning.

\subsection{Endpoint grammar and interaction prediction}
The waypoint experiments use
\begin{equation}
h_t^{\rm end}(y_t)=[z_t;y_t;z_t\odot y_t;A_\omega(a_{t-1})].
\label{eq:actor_features_endpoint}
\end{equation}
For an endpoint $q_t=z_t+\delta_t$,
$z_t\odot q_t=z_t\odot z_t+z_t\odot\delta_t$.  The controlled A--G grammar
study distinguishes a matched displacement/product pair from changing only
the product slot.  If the product acts as a state--intent interaction, paired
Direct scores obey the finite-capacity prediction
\begin{equation}
\begin{aligned}
&\Delta_{\rm match}=(S_C-S_D)-(S_A-S_B)>0,\\
&S_E-S_F>0,
\end{aligned}
\label{eq:product_prediction}
\end{equation}
while $S_C>S_G$ rejects the explanation that centering the product alone is
sufficient.  The complete experiment is in \cref{app:design_selection}.

\subsection{General endpoint conditioning}
For heterogeneous tasks, a more general endpoint design may expose an optional
relation $r$,
\begin{equation}
\mathcal L_{\rm act}^{\tau}(\nu_\tau)
=\mathbb E_{(z,a),(y,r)\sim\nu_\tau}
[-\log p_{\eta_\tau}(a\mid z,y,r,a_{t-1})].
\label{eq:general_endpoint}
\end{equation}
The final raw-goal policy is not explicitly horizon-conditioned: training spans
remaining offsets $r=1,\ldots,5$, but its input contains only the current state,
raw goal displacement, interaction feature, and previous action.

\section{Design Selection and Ablation Logic}
\label{app:design_selection}

This appendix separates model definition from design selection. Each study is
organized as a question, a matched intervention, and the conclusion used to
fix the final configuration; none changes the headline evaluation protocol.

\subsection{Which intent coordinate should the shared operator receive?}
The waypoint endpoint $q_t$, centered waypoint displacement $q_t-z_t$, and
full goal displacement $z_g-z_t$ successively remove an unnecessary absolute
coordinate and the hand-set equal-progress scale. Product-off controls test
whether the gain comes from the coordinate or its state-dependent interaction;
G centers only the interaction while leaving the endpoint main effect intact.
All cells use one full-data pass, three training seeds, the same evaluation
episodes, and deterministic Math-SDPA.

\begin{table}[t]
\centering
\scriptsize
\caption{PushT intent-coordinate grammar under the matched three-seed protocol.
Values are official Direct SR (\%).}
\label{tab:intent_grammar}
\begin{adjustbox}{max width=\columnwidth}
\begin{tabular}{clcr}
\toprule
ID & Main condition & Interaction & Direct SR \\
\midrule
A & $q_t$ & $z_t\odot q_t$ & $78.78\!\pm\!0.96$ \\
B & $q_t$ & omitted & $80.56\!\pm\!0.96$ \\
G & $q_t$ & $z_t\odot(q_t-z_t)$ & $79.67\!\pm\!1.76$ \\
C & $q_t-z_t$ & $z_t\odot(q_t-z_t)$ & $83.11\!\pm\!1.39$ \\
D & $q_t-z_t$ & omitted & $81.00\!\pm\!1.45$ \\
F & $z_g-z_t$ & omitted & $83.22\!\pm\!0.51$ \\
E & $z_g-z_t$ & $z_t\odot(z_g-z_t)$ & \best{$85.78\!\pm\!1.84$} \\
\bottomrule
\end{tabular}
\end{adjustbox}
\end{table}

The matched endpoint-to-centered change A$\rightarrow$C is +4.33 points and
the centered-to-full-goal change C$\rightarrow$E is +2.67 points; A$\rightarrow$E
is +7.00 points (exact $p=4.17\times10^{-9}$). The interaction prediction also
closes quantitatively: the matched difference-in-differences is +3.89 points,
$S_E-S_F=+2.56$, and $S_C-S_G=+3.44$. G shows that merely deleting the
$z_t\odot z_t$ component is insufficient. The supported inductive bias is the
matched intent-dependent sub-grammar $[z_t,m_t,z_t\odot m_t]$ within the full
four-slot input; the shared action-history slot $A_\omega(a_{t-1})$ remains
unchanged. Here $m_t$ explicitly carries the intended displacement.

\begin{figure*}[t]
\centering
\includegraphics[width=0.99\textwidth]{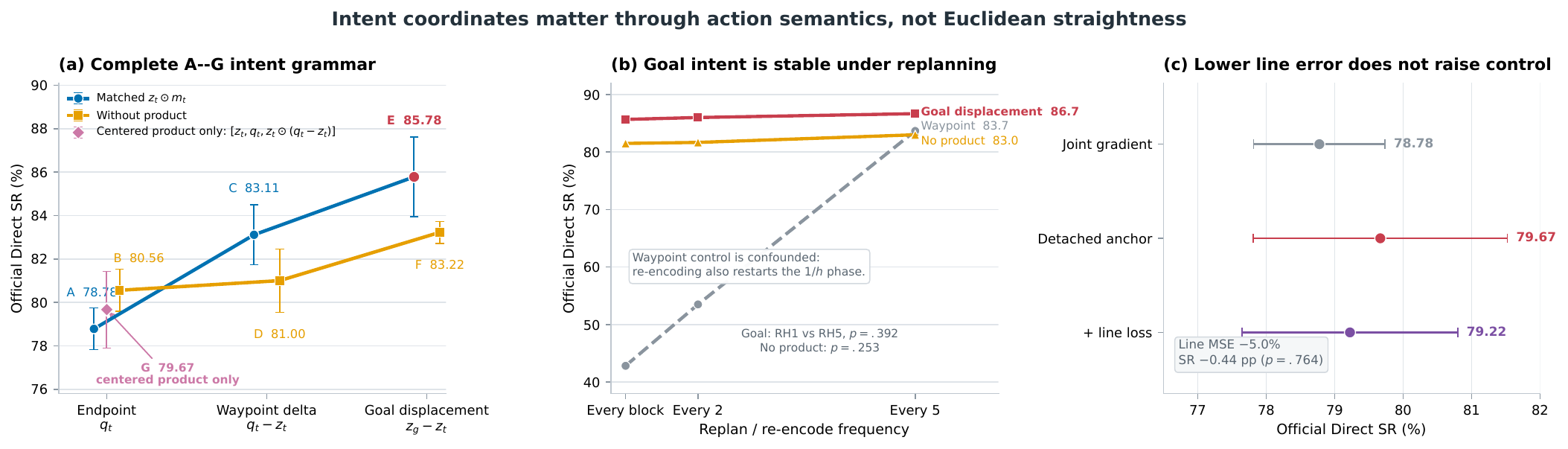}
\caption{\textbf{Controls used to select the intent grammar.} (a) Coordinate
and interaction interventions. (b) Re-encoding every 1, 2, or 5 action blocks.
(c) Detaching the goal anchor and explicitly reducing chord error separate
action-field learning from Euclidean latent straightening. Error bars are
training-seed standard deviations.}
\label{fig:single_condition_geometry}
\end{figure*}

\subsection{Does the result require trajectory straightening or a particular optimizer?}
For the selected full-goal variants, re-encoding intervals 1/2/5 give
85.67/86.00/86.67\% for E and 81.50/81.67/83.00\% for F; the endpoint
differences are not significant ($p=0.392/0.253$). The waypoint curve is not a
clean drift test because every reset also restarts its $1/h$ phase.

Joint goal-intent gradients reach 78.78\%; detaching that branch reaches
79.67\%. Adding a weight-0.01 strict-dynamic SmoothL1 chord penalty lowers
validation chord MSE from 0.2173 to 0.2064 but gives 79.22\% SR
($p=0.764$). A stronger flexible linearity weight of 0.3 collapses Direct SR to
14.11\%. Geometry can therefore become visibly straighter without becoming a
better controller.

The waypoint SIGReg sweep gives 78.56/77.67/76.56\% at 0.02/0.03/0.04 and
68.22\% at 0.09. We use the stable low-regularization region rather than claim
monotonic benefit. A single-seed Huber screening run obtains 55.0\%, whereas
the formal probabilistic NLL settings obtain $68.22\!\pm\!1.26$\% and
$77.67\!\pm\!0.88$\%; protocols are not paired, so this only rejects the claim
that pointwise regression is already known to be preferable. Forward weights
1.0/0.3/0.1/0 give 77.67/73.11/59.56/25.22\%, confirming that the action
interface complements rather than replaces world prediction.

\subsection{Why is Guarded A the optional verifier?}
\Cref{fig:planner_matrix} visualizes the parameter-selection audit; the full
576-job table is in \cref{tab:planner_matrix_full}. PushT benefits from a finite
correction radius, rising from 85.78\% Direct to 92.33\% at $\sigma_0=0.25$,
whereas Reacher degrades at larger radii. Guarded A retains the Direct reference
and global best, reaches 96.86\% macro with 384 sampled sequences, and is
statistically near the simpler unguarded $\sigma_0=0.25$ rule at 96.78\%.
Doubling to $K=256$ reaches 96.83\% and provides no resolved gain. Hence Direct
is the primary interface and Guarded A is a bounded verifier, not a replacement
controller.

\begin{figure*}[t]
\centering
\includegraphics[width=0.99\textwidth]{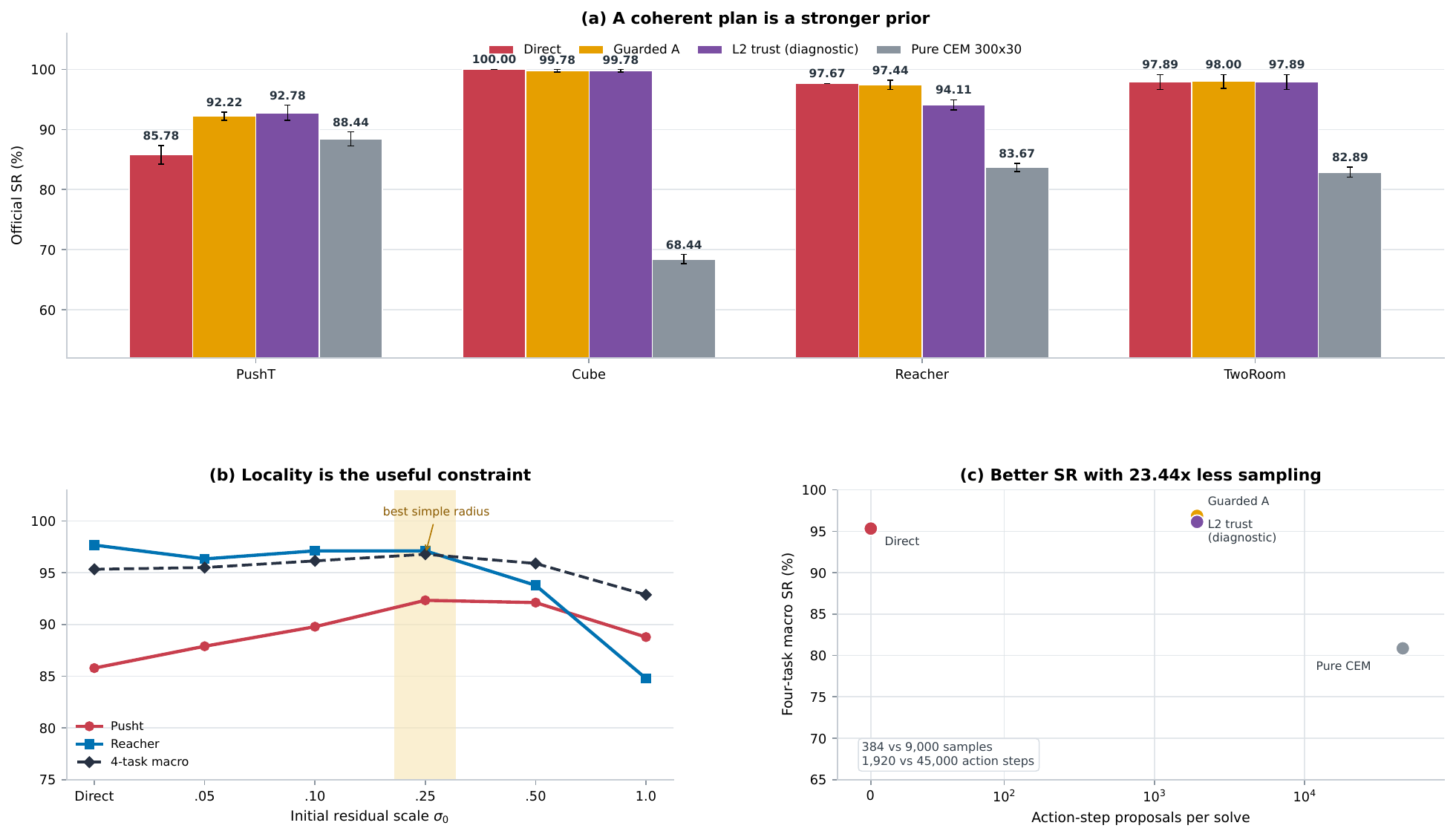}
\caption{\textbf{Verifier-selection controls.} The complete three-seed audit
compares Direct, guarded and unguarded plan-centered search, trust penalties,
actor particles, and actor-disabled CEM.}
\label{fig:planner_matrix}
\end{figure*}

\subsection{Qualitative multi-task geometry}
The fixed-sample projections below illustrate collapse, partial expansion, and
healthy SIGReg-supported spread. A spherical or shell-like t-SNE projection is
not a success criterion: it is a nonlinear view of a high-dimensional cloud
and is reported only to diagnose gross failure morphology.

\begin{figure*}[t]
\centering
\includegraphics[width=0.99\textwidth]{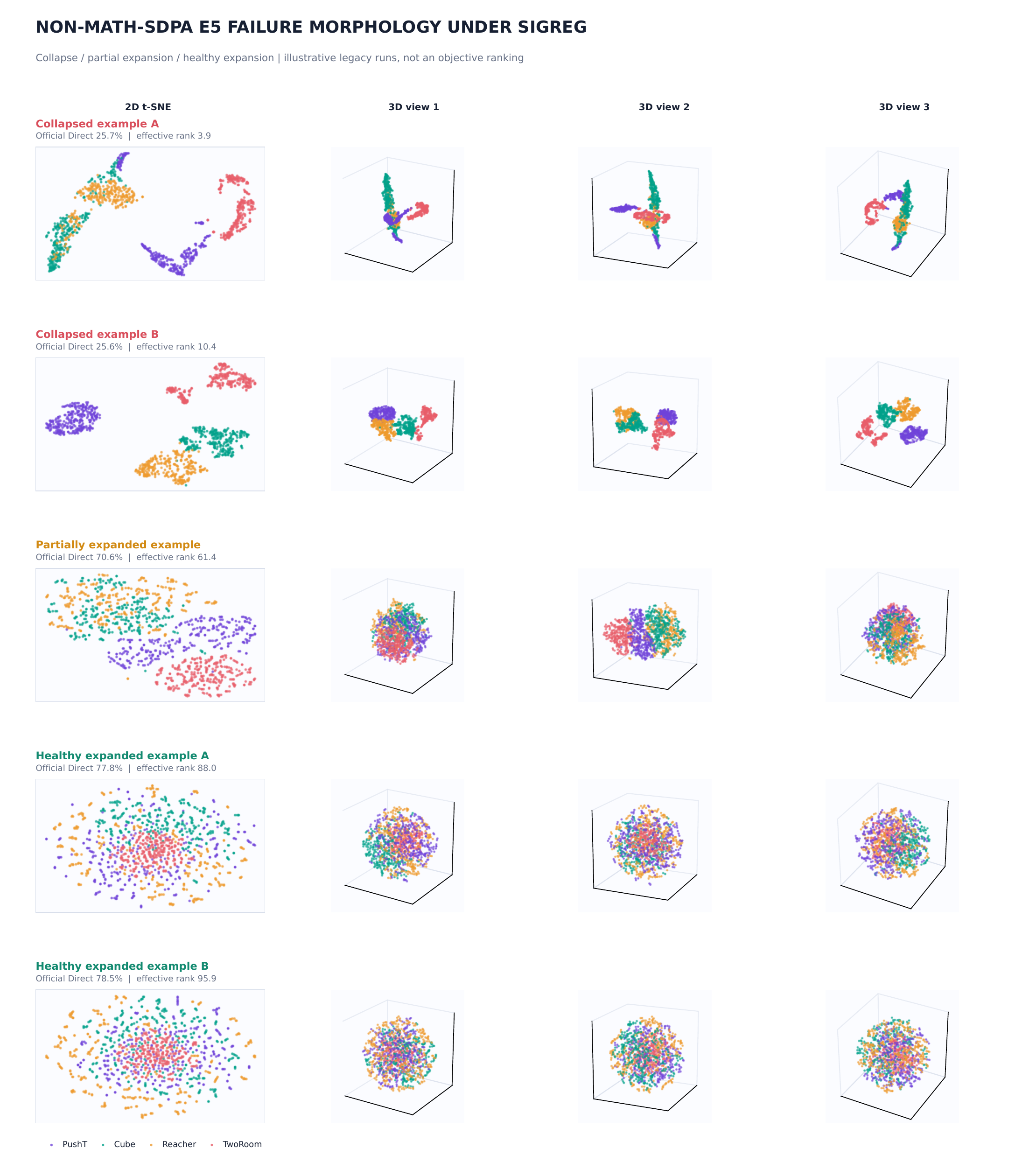}
\caption{Exploratory E5 collapse morphologies from fixed samples. These legacy
non-Math-SDPA runs are qualitative diagnostics and do not enter performance
ranking or headline correlations.}
\label{fig:multitask_geometry}
\end{figure*}

Across these controls, the selection rule is consistent: preserve the world
objective, expose a distance-preserving intent coordinate and its matched
state interaction, learn both condition families through one probabilistic
operator, and add search only as bounded local verification.

\section{Extended Controlled Results}
\label{app:extended_results}

This section preserves the complete protocol tables and selection evidence
summarized in the main text.  It contains no additional model selection based
on test outcomes: the single-task grammar and verifier were locked before the
final matrices, and every multi-task row uses the audited deterministic
Math-SDPA contract.

\subsection{Single-task evaluator and attribution details}
\begin{table*}[t]
\centering
\scriptsize
\caption{\textbf{Official versus CLEAR-LeWM v0.5.1 Moderate single-task SR
(\%).} Every row evaluates the same three one-epoch goal-displacement
checkpoints per task; each checkpoint averages evaluation seeds 0/1/42 with
100 episodes per seed.  Values are mean $\pm$ sample standard deviation across
training seeds.  Protocol scores are never pooled.}
\label{tab:clear_main}
\begin{adjustbox}{max width=\textwidth}
\begin{tabular}{llrrrrr}
\toprule
Protocol & Deployment & PushT & Cube & Reacher & TwoRoom & Macro \\
\midrule
Official & Pure CEM $300{\times}30$ & $88.44\!\pm\!1.17$ & $68.44\!\pm\!.77$ & $83.67\!\pm\!.67$ & $82.89\!\pm\!.84$ & $80.86\!\pm\!.51$ \\
Official & Direct & $85.78\!\pm\!1.54$ & $100.00\!\pm\!.00$ & $97.67\!\pm\!.00$ & $97.89\!\pm\!1.26$ & $95.33\!\pm\!.58$ \\
\best{Official} & \best{Guarded A $128{\times}3$} & \best{$92.22\!\pm\!.69$} & \best{$99.78\!\pm\!.19$} & \best{$97.44\!\pm\!.77$} & \best{$98.00\!\pm\!1.15$} & \best{$96.86\!\pm\!.38$} \\
\midrule
CLEAR Moderate & Pure CEM $300{\times}30$ & $90.22\!\pm\!2.27$ & $49.44\!\pm\!3.60$ & $50.11\!\pm\!1.17$ & $72.11\!\pm\!2.79$ & $65.47\!\pm\!.77$ \\
CLEAR Moderate & Direct & $88.22\!\pm\!.39$ & $99.89\!\pm\!.19$ & $49.56\!\pm\!2.52$ & $95.67\!\pm\!2.33$ & $83.33\!\pm\!1.00$ \\
\best{CLEAR Moderate} & \best{Guarded A $128{\times}3$} & \best{$92.78\!\pm\!1.71$} & \best{$99.44\!\pm\!.20$} & \best{$51.67\!\pm\!3.18$} & \best{$95.44\!\pm\!2.22$} & \best{$84.83\!\pm\!1.32$} \\
\bottomrule
\end{tabular}
\end{adjustbox}
\end{table*}

\begin{figure*}[t]
\centering
\includegraphics[width=0.98\textwidth]{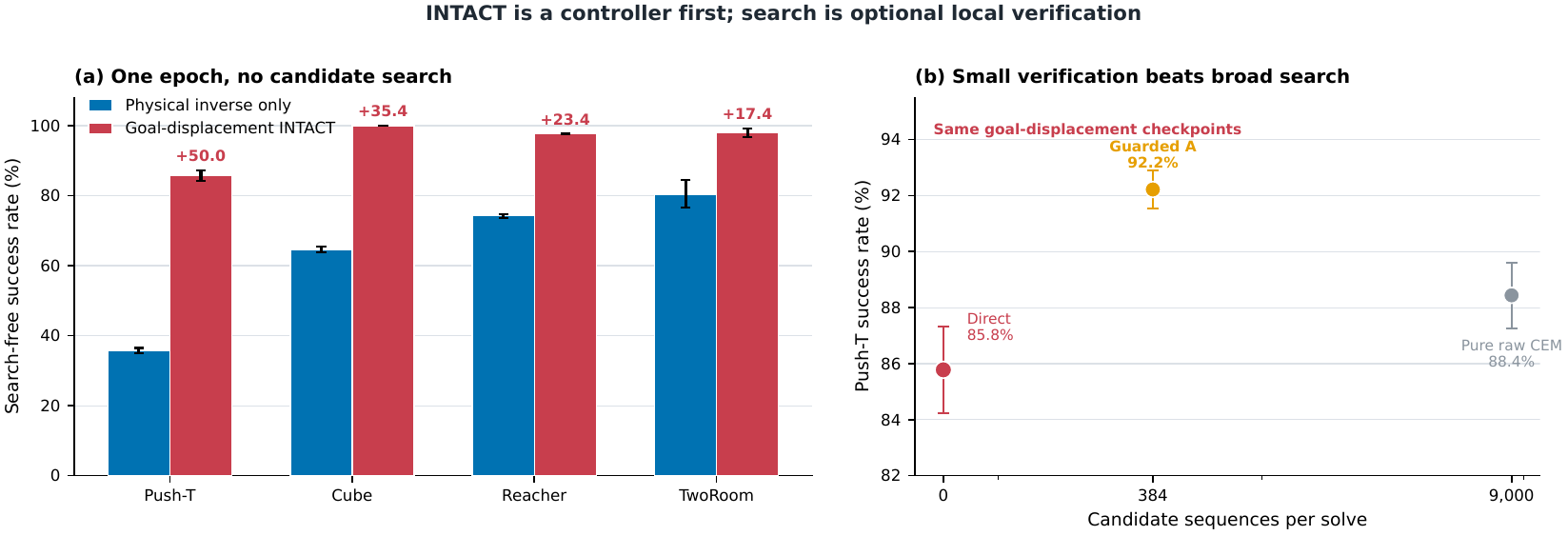}
\caption{\textbf{Search-free control and optional verification.} Left:
physical-successor inversion versus goal-displacement INTACT after one epoch.
Right: Direct uses no candidate sequence, Guarded A locally verifies 384, and
broad actor-disabled CEM evaluates 9,000.  More search is not automatically
better once the action conditional is aligned.}
\label{fig:query_direct}
\end{figure*}

\begin{table}[t]
\centering
\small
\caption{PushT waypoint attribution chain.  Pure CEM disables the
intent-conditioned actor; probes are frozen and episode-disjoint.  The final
row also changes SIGReg and is not an intent-only attribution.}
\label{tab:single_chain}
\begin{adjustbox}{max width=\columnwidth}
\begin{tabular}{lrrrrr}
\toprule
Training & Direct & CEM $30{\times}10$ & State $R^2$ & Action $R^2$ & Rank \\
\midrule
LeWM & -- & $42.22\!\pm\!9.50$ & $.616\!\pm\!.008$ & $.673\!\pm\!.062$ & $63.7\!\pm\!.9$ \\
$+$ physical inverse & $35.78\!\pm\!.69$ & $57.67\!\pm\!2.52$ & $.652\!\pm\!.010$ & $.791\!\pm\!.003$ & $67.0\!\pm\!.8$ \\
$+$ matched waypoint intent & $68.22\!\pm\!1.26$ & $61.44\!\pm\!2.04$ & $.658\!\pm\!.003$ & $.786\!\pm\!.003$ & $67.0\!\pm\!.4$ \\
Waypoint INTACT, SIGReg $0.03$ & \best{$77.67\!\pm\!.88$} & \best{$69.44\!\pm\!1.64$} & \best{$.680\!\pm\!.005$} & \best{$.804\!\pm\!.002$} & $43.7\!\pm\!.6$ \\
\bottomrule
\end{tabular}
\end{adjustbox}
\end{table}

The Forward Predictor remains necessary: weights 1.0/0.3/0.1/0 give PushT
Direct SR 77.67/73.11/59.56/25.22\%.  Increasing the one-epoch peak learning
rate to $7.5\times10^{-4}$ gives $77.00\!\pm\!2.03$\%; a second epoch gives
$77.22\!\pm\!1.07$\%, versus $77.67\!\pm\!.88$ at the selected setting.  The
second epoch raises effective rank from about 43.7 to 58.8 without raising SR,
directly separating geometric spread from action-aligned representation
quality.

\subsection{Complete shared-encoder low-budget result}
\begin{table*}[t]
\centering
\small
\caption{\textbf{Controlled Math-SDPA shared-encoder E1 result.} Values are
mean $\pm$ sample standard deviation over three training seeds; each seed
averages official evaluation seeds $\{0,1,42\}$ with 100 episodes each.  LeWM
has no INTACT Predictor and uses CEM $300{\times}30$; other rows use Direct.}
\label{tab:multitask_core}
\begin{adjustbox}{max width=\textwidth}
\begin{tabular}{llcrrrrr}
\toprule
Training cell & Evaluation & $n$ & PushT & Cube & Reacher & TwoRoom & Macro \\
\midrule
LeWM & CEM $300{\times}30$ & 3 & $17.00\!\pm\!7.17$ & $52.89\!\pm\!3.95$ & \best{$25.00\!\pm\!4.26$} & $34.33\!\pm\!1.76$ & $32.31\!\pm\!2.34$ \\
Inverse only & Direct & 3 & $24.89\!\pm\!5.97$ & $48.22\!\pm\!3.83$ & $9.44\!\pm\!3.01$ & $51.33\!\pm\!13.32$ & $33.47\!\pm\!3.98$ \\
Waypoint intent only & Direct & 3 & $17.67\!\pm\!2.40$ & $68.78\!\pm\!5.50$ & $11.89\!\pm\!2.78$ & $44.67\!\pm\!31.13$ & $35.75\!\pm\!7.58$ \\
Goal intent only & Direct & 3 & $24.00\!\pm\!4.36$ & \best{$73.22\!\pm\!3.67$} & $11.56\!\pm\!2.04$ & $48.22\!\pm\!17.20$ & $39.25\!\pm\!5.40$ \\
\best{Waypoint INTACT} & \best{Direct} & 3 & \best{$30.56\!\pm\!9.34$} & $69.89\!\pm\!10.99$ & $10.33\!\pm\!5.03$ & \best{$60.00\!\pm\!10.58$} & \best{$42.69\!\pm\!7.19$} \\
Goal INTACT & Direct & 3 & $26.89\!\pm\!6.26$ & $61.44\!\pm\!10.46$ & $10.33\!\pm\!3.18$ & $48.67\!\pm\!21.31$ & $36.83\!\pm\!5.63$ \\
\bottomrule
\end{tabular}
\end{adjustbox}
\end{table*}

Waypoint INTACT is the strongest first-pass cell, whereas goal-displacement
INTACT becomes the strongest E5 cell.  This reversal is why the E1 table is a
low-budget optimization diagnostic rather than the final objective ranking.

\subsection{Matched E5 inference decomposition}
\begin{table*}[t]
\centering
\scriptsize
\setlength{\tabcolsep}{4.2pt}
\renewcommand{\arraystretch}{0.96}
\caption{\textbf{Matched E5 inference decomposition.} Values are mean $\pm$
sample standard deviation over three training seeds.  Panel (a) removes all
learned action heads; panel (b) locally verifies the learned plan; panel (c)
uses CLEAR Moderate with each cell's native interface.}
\label{tab:multitask_inference_decomp}

\textbf{(a) Official actor-disabled pure CEM $300{\times}30$}\par\vspace{0.15em}
\begin{adjustbox}{max width=\textwidth}
\begin{tabular}{lrrrrr}
\toprule
Training cell & PushT & Cube & Reacher & TwoRoom & Macro \\
\midrule
LeWM & $74.56\!\pm\!3.67$ & $67.33\!\pm\!1.86$ & $83.11\!\pm\!.96$ & $39.67\!\pm\!8.39$ & $66.17\!\pm\!2.67$ \\
Inverse only & $75.78\!\pm\!1.02$ & $66.78\!\pm\!2.36$ & $81.78\!\pm\!1.64$ & $51.44\!\pm\!1.07$ & $68.94\!\pm\!.43$ \\
Waypoint intent only & $78.44\!\pm\!4.72$ & $68.67\!\pm\!1.76$ & $83.22\!\pm\!2.69$ & $39.11\!\pm\!8.17$ & $67.36\!\pm\!4.06$ \\
Waypoint INTACT & $78.22\!\pm\!3.02$ & $69.56\!\pm\!2.80$ & $83.00\!\pm\!2.19$ & $46.44\!\pm\!3.34$ & $69.31\!\pm\!.68$ \\
Goal intent only & \best{$80.78\!\pm\!2.36$} & $68.00\!\pm\!3.18$ & \best{$84.67\!\pm\!1.33$} & $40.11\!\pm\!8.30$ & $68.39\!\pm\!2.22$ \\
\best{Goal INTACT} & $78.11\!\pm\!1.17$ & \best{$70.33\!\pm\!3.00$} & $80.33\!\pm\!2.03$ & \best{$51.56\!\pm\!4.53$} & \best{$70.08\!\pm\!1.13$} \\
\bottomrule
\end{tabular}
\end{adjustbox}

\vspace{0.25em}
\textbf{(b) Official Guarded A $128{\times}3$}\par\vspace{0.15em}
\begin{adjustbox}{max width=\textwidth}
\begin{tabular}{lrrrrr}
\toprule
Training cell & PushT & Cube & Reacher & TwoRoom & Macro \\
\midrule
Inverse only & $55.67\!\pm\!4.10$ & $65.89\!\pm\!1.02$ & $78.67\!\pm\!1.20$ & $66.33\!\pm\!6.08$ & $66.64\!\pm\!.88$ \\
Waypoint intent only & $74.67\!\pm\!3.79$ & \best{$99.78\!\pm\!.19$} & $87.78\!\pm\!1.07$ & $74.78\!\pm\!10.63$ & $84.25\!\pm\!3.58$ \\
Waypoint INTACT & $83.22\!\pm\!2.36$ & $99.11\!\pm\!.51$ & $85.33\!\pm\!.58$ & $74.89\!\pm\!7.15$ & $85.64\!\pm\!2.04$ \\
Goal intent only & $80.44\!\pm\!.96$ & $99.33\!\pm\!.33$ & $94.67\!\pm\!.33$ & $69.33\!\pm\!5.93$ & $85.94\!\pm\!1.57$ \\
\best{Goal INTACT} & \best{$85.00\!\pm\!2.85$} & $99.00\!\pm\!.00$ & \best{$97.89\!\pm\!.69$} & \best{$80.00\!\pm\!4.41$} & \best{$90.47\!\pm\!.84$} \\
\bottomrule
\end{tabular}
\end{adjustbox}

\vspace{0.25em}
\textbf{(c) CLEAR-LeWM v0.5.1 Moderate}\par\vspace{0.15em}
\begin{adjustbox}{max width=\textwidth}
\begin{tabular}{lrrrrr}
\toprule
Training cell & PushT & Cube & Reacher & TwoRoom & Macro \\
\midrule
LeWM (pure CEM) & $78.78\!\pm\!2.52$ & $51.33\!\pm\!1.20$ & $46.00\!\pm\!4.06$ & $33.33\!\pm\!13.86$ & $52.36\!\pm\!2.97$ \\
Inverse only & $32.67\!\pm\!.88$ & $42.67\!\pm\!1.67$ & $31.67\!\pm\!2.33$ & $70.11\!\pm\!2.17$ & $44.28\!\pm\!1.33$ \\
Waypoint intent only & $60.89\!\pm\!2.91$ & \best{$99.78\!\pm\!.19$} & $35.78\!\pm\!7.58$ & $67.11\!\pm\!12.01$ & $65.89\!\pm\!2.79$ \\
Waypoint INTACT & $74.78\!\pm\!4.34$ & $99.11\!\pm\!.84$ & $32.22\!\pm\!5.43$ & $73.33\!\pm\!3.28$ & $69.86\!\pm\!1.06$ \\
Goal intent only & $73.00\!\pm\!3.18$ & \best{$100.00\!\pm\!.00$} & $39.44\!\pm\!2.71$ & $60.67\!\pm\!5.51$ & $68.28\!\pm\!1.17$ \\
\best{Goal INTACT} & \best{$84.33\!\pm\!2.00$} & $99.11\!\pm\!1.54$ & \best{$47.44\!\pm\!2.91$} & \best{$79.33\!\pm\!5.03$} & \best{$77.56\!\pm\!.39$} \\
\bottomrule
\end{tabular}
\end{adjustbox}
\end{table*}

\subsection{Representation and conversion diagnostics}
\begin{table*}[t]
\centering
\scriptsize
\caption{\textbf{Controlled shared-encoder representation snapshots.}
$\bar\sigma$ is mean per-dimension latent standard deviation, cosine is mean
pairwise cosine, and action $R^2$ is an episode-disjoint frozen probe.}
\label{tab:multitask_diagnostics}
\label{tab:multitask_alignment}
\begin{minipage}[t]{0.495\textwidth}
\centering
\textbf{E1}\par\vspace{0.1em}
\begin{tabular}{lcrrrr}
\toprule
Training cell & $n$ & Rank & $\bar\sigma$ & Cos. & Act. $R^2$ \\
\midrule
LeWM & 3 & 25.53 & .830 & .554 & .309 \\
Inverse only & 3 & 24.03 & .902 & .590 & .372 \\
Waypoint intent only & 3 & 24.83 & .846 & .582 & .328 \\
Goal intent only & 3 & 24.45 & .839 & .604 & .349 \\
Waypoint INTACT & 3 & \best{25.73} & .877 & .587 & \best{.383} \\
Goal INTACT & 3 & 25.34 & .888 & .616 & .382 \\
\bottomrule
\end{tabular}
\end{minipage}\hfill
\begin{minipage}[t]{0.495\textwidth}
\centering
\textbf{E5}\par\vspace{0.1em}
\begin{tabular}{lcrrrr}
\toprule
Training cell & $n$ & Rank & $\bar\sigma$ & Cos. & Act. $R^2$ \\
\midrule
LeWM & 3 & \best{94.06} & .964 & .039 & .366 \\
Inverse only & 3 & 85.24 & .956 & .087 & .401 \\
Waypoint intent only & 3 & 93.87 & .959 & .049 & .393 \\
Goal intent only & 3 & 93.47 & .965 & .047 & .390 \\
Waypoint INTACT & 3 & 87.17 & .960 & .080 & .406 \\
Goal INTACT & 3 & 89.26 & .958 & .081 & \best{.409} \\
\bottomrule
\end{tabular}
\end{minipage}
\end{table*}

Across 45 formula-eligible checkpoints, pooled predicted--expert CKA/kNN
correlations with official Direct SR are $.897/.954$; pointwise action $R^2$
is $.815$ and NLL is $-.786$.  Leave-one-epoch-out values remain
$[.888,.914]$ for CKA and $[.946,.963]$ for kNN.  These analyses use one point
per checkpoint, not one point per transition.  Intent-only, inverse-only, and
LeWM cells are excluded because a transition must evaluate both local and goal
likelihoods through the same actor to instantiate the claimed family map.

Fixed-sample t-SNE is retained only as a qualitative diagnostic.  Under
SIGReg, healthy high-dimensional clouds can project to an interleaved shell;
this appearance neither computes nor ranks SR.  The quantitative analyses use
original-space rank, standard deviation, forward error, probes, and
predicted--expert action-family relations.

\subsection{Task-conditioned gauge controls}
Correct pairing raises CLEAR Moderate SR from 9.46\% shuffled to 68.04\%, a
$+58.58$ point gain (8-cluster bootstrap 95\% CI $[48.83,67.58]$), positive
for all 24 protocol units.  Calibration with 4/16/64/full episodes gives
26.67/56.00/64.42/64.83\% SR.  Same-objective cross-seed swaps define the
alignment ceiling; LeWM and inverse backbones aligned into a trained Full actor
are positive alignability controls.  In reverse, Full-backbone-to-LeWM-actor
stays at 5.08\%, and the inverse actor recovers only inverse-level function.

\begin{figure*}[t]
\centering
\includegraphics[width=0.995\textwidth]{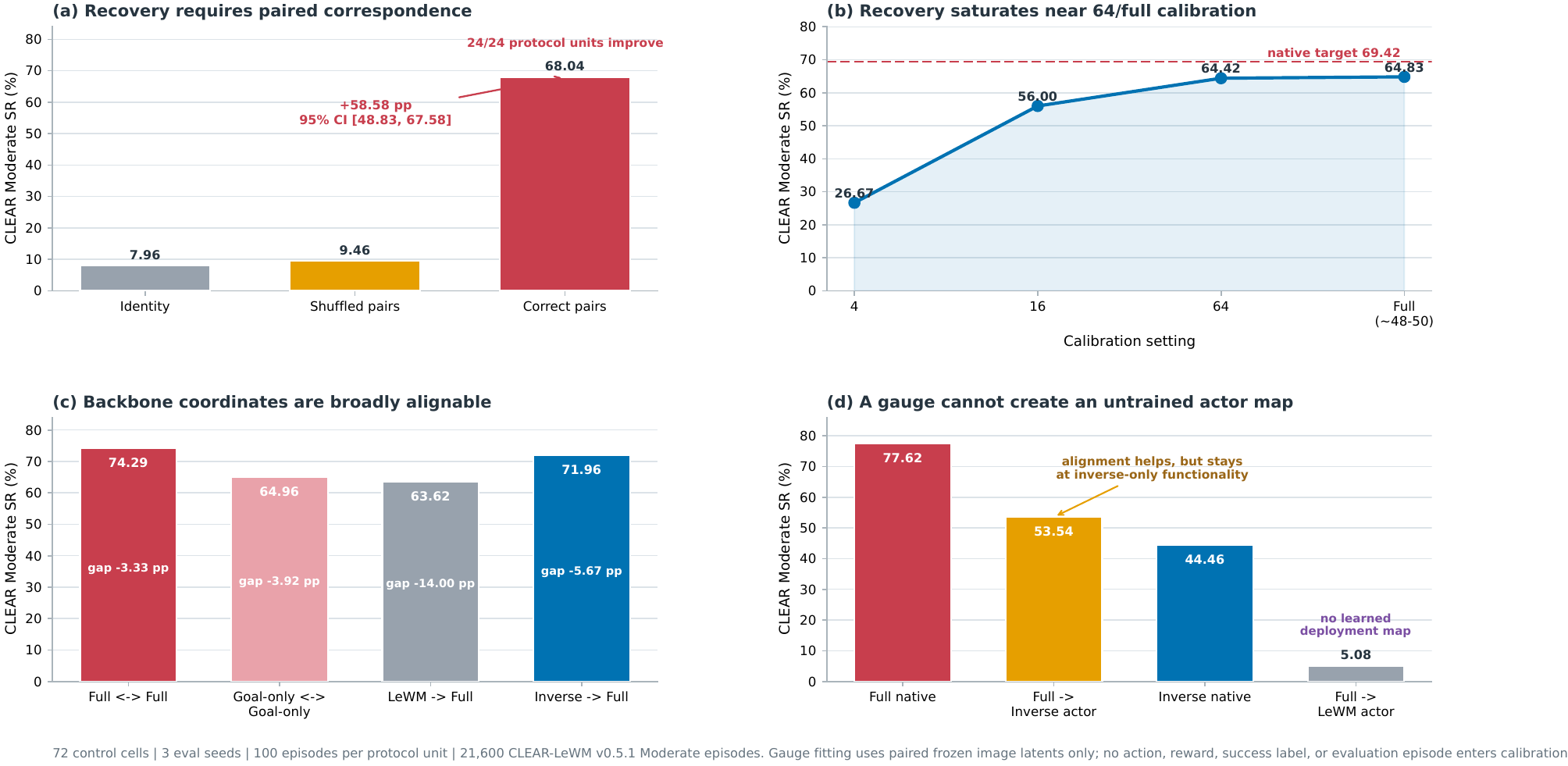}
\caption{\textbf{Gauge controls isolate correspondence from actor
functionality.} Correctly paired calibration restores control; recovery
saturates near 64/full episodes; same-objective swaps define a ceiling; reverse
swaps show that a coordinate map cannot create an untrained deployment
conditional.  The 72-cell audit contains 21,600 CLEAR Moderate episodes and
excludes all evaluation episodes from calibration.}
\label{fig:task_conditioned_gauge}
\end{figure*}

Cross-task and leave-one-task-out actor-output transfer remain negative, while
pooled calibration reaches only $.484/.539$ $R^2$.  The shared visual backbone
can therefore support several task-local charts; the task-specific heads do not
require one global canonical action space.

\section{Actor-Sharing Mechanism Audit}
\label{app:actor_sharing}
\begin{table*}[t]
\centering
\small
\caption{\textbf{Sharing is the operative inductive bias.} Matched PushT
waypoint Direct SR (\%).  Except for the explicit double-capacity control,
parameter counts match within 0.12\%.  $p$ is an episode-paired exact McNemar
comparison with the fully shared model.}
\label{tab:actor_sharing}
\begin{adjustbox}{max width=\textwidth}
\begin{tabular}{lrrrr}
\toprule
Action-conditional design & Params (M) & Direct SR & $\Delta$ vs shared & Paired $p$ \\
\midrule
Waypoint intent only & 3.110 & $65.56\!\pm\!1.35$ & $-11.33$ & $4.23\!\times\!10^{-15}$ \\
Independent, parameter matched & 3.114 & $71.22\!\pm\!2.27$ & $-5.67$ & $1.15\!\times\!10^{-5}$ \\
Independent, double capacity & 6.220 & $72.89\!\pm\!2.17$ & $-4.00$ & 0.00362 \\
Shared trunk, two outputs & 3.110 & $74.67\!\pm\!1.53$ & $-2.22$ & 0.0927 \\
Condition-token shared & 3.111 & $77.11\!\pm\!2.12$ & $+0.22$ & 0.923 \\
\best{Fully shared (default)} & 3.110 & $76.89\!\pm\!.69$ & 0 & -- \\
\bottomrule
\end{tabular}
\end{adjustbox}
\end{table*}

The closed-loop result is paired with mechanism measurements on fixed PushT waypoint batches.
``Shared fraction'' is the fraction of action-network parameters touched by
both local and deployment losses. Gradient cosine is computed in the complete
action-network parameter coordinates; it is zero by construction when the two
actors are disjoint. Shuffle response is the normalized-action MAE between the
original goal-conditioned output and the output after permuting goals across
the batch. Consistency is
$\|F_\phi(z_t,\mu_\eta(z_t,q_t))-q_t\|_2^2$.

\begin{table*}[t]
\centering
\small
\caption{Mechanism diagnostics for the matched PushT actor-sharing ablation.
These measurements explain the closed-loop table but are not substitutes for
SR.}
\label{tab:actor_sharing_diagnostics}
\begin{adjustbox}{max width=\textwidth}
\begin{tabular}{lrrrr}
\toprule
Action-conditional design & Shared fraction & Grad. cosine & Shuffle-goal MAE & Consistency MSE \\
\midrule
Condition-token shared & 1.000 & 0.506 & 1.256 & 0.1320 \\
Fully shared & 1.000 & \best{0.575} & \best{1.263} & 0.1275 \\
Shared trunk, two outputs & 0.987 & 0.138 & 1.200 & 0.1309 \\
Independent, double capacity & 0 & 0 & 1.030 & 0.1382 \\
Independent, parameter matched & 0 & 0 & 1.029 & 0.1370 \\
Goal-intent only & 1.000 & 0.321 & 0.889 & \best{0.1133} \\
\bottomrule
\end{tabular}
\end{adjustbox}
\end{table*}

The fully shared model receives positively aligned physical/deployment
gradients (cosine 0.575). As sharing is weakened, gradient coordination and SR
decline together. Across the six designs, shuffle-goal response has exploratory
Spearman $\rho=0.943$ with SR ($n=6$), indicating that stronger controllers use
the goal condition more actively; the small design count precludes a universal
claim. Two negative controls are especially informative. Goal-intent-only has
the lowest consistency MSE but the worst SR, so self-consistency with the
Forward Predictor does not establish physical correctness. Conversely, the
double-capacity independent model obtains real/deployment NLLs of
$-1.347/-1.160$ yet reaches only 72.89\% SR. Separately fitting both likelihoods
is therefore not equivalent to learning one deployable inverse map.

\section{Exploratory Diagnostic Definitions and Correlations}
\label{app:diagnostics}

For centered latent samples $Z\in\R^{n\times D}$ with normalized covariance
eigenvalues $\bar\lambda_j$, entropy effective rank is
\begin{equation}
d_{\mathrm{eff}}=\exp\!\left(-\sum_j\bar\lambda_j
\log(\bar\lambda_j+\epsilon)\right).
\end{equation}
Let $r_i=\lVert Z_i-\bar Z\rVert_2$, let $\mathrm{CV}_r$ be the coefficient of
variation of these radii, and let $\mathrm{CV}_{\chi(d)}$ be that of a chi
distribution with $d$ degrees of freedom. We define
\begin{align}
S_{\rm rank}&=d_{\rm eff}/D,\\
S_{\rm shell}&=\exp\!\left[-\left|\log
\frac{\mathrm{CV}_r}{\mathrm{CV}_{\chi(d_{\rm eff})}}\right|\right],\\
\mathrm{SRS}&=100\sqrt{S_{\rm rank}S_{\rm shell}}.
\label{eq:srs_definition}
\end{align}
SRS is high when the spectrum uses many dimensions and sample radii resemble
the Gaussian shell expected at that effective dimension; it does not measure
what information those dimensions encode.

For the actor diagnostic, write $(\mu_t^{\rm dep},\sigma_t^{\rm dep})$ for the
Gaussian under the deployed goal condition and
$(\widetilde\mu_t,\widetilde\sigma_t)$ for a deterministic different-goal
condition on the same current state. The reported condition response is the
per-sample RMS standardized mean shift
\begin{equation}
G_{\rm sep}=\mathbb E_t\!\left[
\left(\frac{1}{d_a}\sum_{j=1}^{d_a}
\frac{(\mu_{t,j}^{\rm dep}-\widetilde\mu_{t,j})^2}
{\tfrac12[(\sigma_{t,j}^{\rm dep})^2+
\widetilde\sigma_{t,j}^{2}]+\epsilon}\right)^{1/2}\right].
\label{eq:gsep_definition}
\end{equation}
Thus $G_{\rm sep}$ asks whether the actor uses its condition relative to its
own predictive uncertainty. It does not ask whether the changed goal requires
a different expert action, nor whether the response has the correct direction.

We pair rank, SRS, and $G_{\rm sep}$ with forward MSE, deployment NLL/RMSE,
symmetric KL between action conditionals, and shuffled-goal $\Delta$NLL. Every
method uses the same 64 sequences per task. Correlations first subtract the
mean within each task and then pool the 56 method--task cells.
These endpoint runs predate the controlled Math-SDPA factorial. Their matched
sampling makes the geometric and conversion metrics informative as exploratory
mechanism diagnostics, but their SR values are not used to rank training
objectives in the main performance comparison.

\begin{table}[t]
\centering
\small
\caption{Task-centered Pearson correlation with E5 Direct SR. These are
diagnostic associations, not independent causal effects.}
\label{tab:metric_correlations}
\begin{tabular}{lrr}
\toprule
Metric & Moderate & Strict \\
\midrule
Effective rank & 0.888 & 0.846 \\
$G_{\mathrm{sep}}$ & 0.819 & 0.731 \\
Deployment NLL & $-0.640$ & $-0.766$ \\
\bottomrule
\end{tabular}
\end{table}

SRS correlates with CLEAR v0.3 E5 Moderate Direct SR at Pearson $r=0.978$
and Spearman $\rho=0.652$, and with Strict Direct SR at $r=0.977$ and
$\rho=0.809$, over the ten semantic endpoint configurations. This smaller analysis is
useful for geometry auditing but is not treated as a universal threshold.

\section{Single-Task Latent-Linearity Stress Test}
\label{app:linearity_control}

\Cref{tab:linearity_control} is a separately trained four-frame PushT isolation,
not a row from the controlled multi-task factorial. Every model uses full data for one
epoch, learning rate $10^{-4}$, batch size 256, physical-inverse weight 0.1,
and SIGReg 0.09; only the latent-linearity term changes. Each cell pools three
training seeds (3072--3074), three evaluation seeds (0, 1, 42), and 100 episodes
per pair, for 900 episodes. The flexible penalty permits a learned local
adjustment, whereas the strict penalty directly targets Euclidean waypoints.
Machine-readable values come from experiment
\nolinkurl{linearity4_exact_isolation_pusht_e1}, specifically the Direct and
pure-CEM summary artifacts.

\begin{table}[t]
\centering
\small
\caption{Separate matched PushT E1 latent-linearity stress test, official SR
(\%). The control is physical-inverse LeWM with no linearity loss, not
forward-only LeWM.}
\label{tab:linearity_control}
\begin{adjustbox}{max width=\columnwidth}
\begin{tabular}{lrrr}
\toprule
Linearity term & Weight & Direct & Pure CEM $30{\times}10$ \\
\midrule
Inverse control (no $\mathcal L_{\rm lin}$) & 0 & 34.00 & \best{43.78} \\
Flexible & 0.01 & \best{35.89} & 41.11 \\
Strict & 0.01 & 35.11 & 41.22 \\
Strict & 0.03 & 35.11 & 41.78 \\
Flexible & 0.30 & 14.11 & -- \\
\bottomrule
\end{tabular}
\end{adjustbox}
\end{table}

The strict losses reduce latent waypoint error but do not improve either
Direct or actor-disabled planning. At flexible weight 0.30, effective rank
falls to 41.96 and the frozen action probe falls to $R^2=0.111$. This negative
control is why INTACT uses the Euclidean construction only to form an actor
condition and never as a target for the encoder trajectory.
Its 34.00\% control should not be compared numerically to the 35.78\%
single-task physical-inverse row in \cref{tab:single_direct}: the stress test
uses an independently trained seed set and a four-frame isolation protocol.

\section{Historical Latent-Perturbation Diagnostic}
\label{app:latent_diffusion}

Before INTACT, we tested whether replacing the deterministic latent predictor
with five-step flow matching improved LeWM. On PushT, the deterministic model
reached 26\%, 86\%, 94\%, and 96\% official SR at E1, E5, E10, and E15. The
best x-prediction/loss-x flow model reached 4\%, 76\%, 88\%, and 92\%; a second
flow run with matched SIGReg 0.09 reached 4\%, 44\%, 72\%, and 76\%. At matched
139,319 optimizer steps, the strongest flow run used 0.568 seconds/step versus
0.479 for the deterministic predictor, an 18.4\% increase.

This historical control is deliberately not a headline ablation: the strongest
flow run changes SIGReg from 0.09 to 0.03, and there is no matched RGB-noise
control. It supports only the narrower observation that perturbing an already
compressed latent did not improve low-budget learning in these runs. This is
consistent with GC-IDM's independent latent-noise ablation, whose taskwise
optimum occurs at zero noise \citep{nguyen2026gcidm}. It does not establish that
latent diffusion is universally harmful or that it is directly comparable to
pixel-space diffusion.

\section{Complete Goal-Displacement Planner Matrix}
\label{app:planner_matrix}

\Cref{tab:planner_matrix_full} reports every completed planner cell used in
\cref{fig:planner_matrix}. Each method contains 36 jobs: three training
checkpoints per task, each evaluated with seeds $\{0,1,42\}$ and 100 episodes.
Macro is the unweighted task mean. The 12 base rules contribute 432 jobs and
the four guarded rules contribute 144, each under a separate PASS audit.

\begin{table*}[t]
\centering
\scriptsize
\caption{Complete 576-job MATH-only planner matrix, official SR (\%). ``Plan''
means that the deterministic goal-displacement action sequence initializes the
factorized raw-action distribution; trust variants use coefficient one.
Guarded rules preserve the Direct reference and globally best observed
sequence across refits.}
\label{tab:planner_matrix_full}
\begin{adjustbox}{max width=\textwidth}
\begin{tabular}{llrrrrrr}
\toprule
Planner family & Setting & PushT & Cube & Reacher & TwoRoom & Macro & Worst \\
\midrule
Direct & Goal-displacement mean & 85.78 & \best{100.00} & \best{97.67} & 97.89 & 95.33 & 85.78 \\
Actor particles & K128-I1 & 91.44 & 99.67 & 82.78 & 98.00 & 92.97 & 82.78 \\
Actor particles & K128-I3 & 91.89 & 99.67 & 78.44 & 97.44 & 91.86 & 78.44 \\
\midrule
Plan + raw CEM & $\sigma_0=0.05$ & 87.89 & \best{100.00} & 96.33 & 97.78 & 95.50 & 87.89 \\
Plan + raw CEM & $\sigma_0=0.10$ & 89.78 & \best{100.00} & 97.11 & 97.67 & 96.14 & 89.78 \\
Plan + raw CEM & $\sigma_0=0.25$ & 92.33 & 99.78 & 97.11 & 97.89 & 96.78 & 92.33 \\
Plan + raw CEM & $\sigma_0=0.50$ & 92.11 & 99.78 & 93.78 & 97.89 & 95.89 & 92.11 \\
Plan + raw CEM & $\sigma_0=1.00$ & 88.78 & 99.56 & 84.78 & \best{98.33} & 92.86 & 84.78 \\
Plan + L2 trust & $\sigma_0=0.50$ & \best{92.78} & 99.78 & 94.11 & 97.89 & 96.14 & \best{92.78} \\
Plan + actor-NLL trust & $\sigma_0=0.50$ & 92.11 & 99.89 & 94.22 & 97.89 & 96.03 & 92.11 \\
\midrule
Guarded B & K256-I3, $\sigma_0=0.25$ & 92.00 & 99.89 & \best{97.67} & 97.78 & 96.83 & 92.00 \\
Guarded C & K128-I5, smooth & 91.56 & 99.89 & 97.56 & 97.89 & 96.72 & 91.56 \\
Guarded D & K128-I3, actor-cov + L2 & 91.67 & 99.89 & 97.22 & 98.00 & 96.69 & 91.67 \\
\best{Guarded A} & \best{K128-I3, $\sigma_0=0.25$} & \best{92.22} & \best{99.78} & \best{97.44} & \best{98.00} & \best{96.86} & \best{92.22} \\
\midrule
Pure raw CEM & K128-I3 & 58.00 & 65.44 & 73.44 & 80.56 & 69.36 & 58.00 \\
Pure raw CEM & K300-I30 & 88.44 & 68.44 & 83.67 & 82.89 & 80.86 & 68.44 \\
\bottomrule
\end{tabular}
\end{adjustbox}
\end{table*}

\section{Episode-Disjoint Anti-Retrieval Audit}
\label{app:episode_disjoint}
The fixed PushT manifest contains 14,948/1,868/1,869 train/validation/test
episodes with zero overlap. Normalization, optimization, and checkpoint
selection exclude the test episodes. We retrain transition-only and full
goal-intent models, then evaluate the same 100 test-only start--goal pairs in
each of three training-seed by three evaluation-seed cells.

\begin{table}[t]
\centering
\small
\caption{PushT Direct SR (\%) on the held-out 10\% episode split. Each cell
contains 100 paired test-only episodes.}
\label{tab:episode_disjoint}
\begin{tabular}{lrrrr}
\toprule
Training objective & eval 0 & eval 1 & eval 42 & pooled \\
\midrule
Transition inverse & 35.67 & 44.33 & 40.67 & 40.22 \\
+ deployment goal intent & 75.33 & 81.33 & 76.00 & \textbf{77.56} \\
\bottomrule
\end{tabular}
\end{table}

The full model reaches \EpisodeQuerySR{}\%, only 0.12 point below the 77.67\%
standard-protocol estimate. Among 900 paired episodes, 358 are solved only by
the goal-intent model and 22 only by transition inverse (descriptive exact
McNemar $p=2.36\times10^{-79}$; checkpoint clustering precludes treating all
episodes as independent). Every training-seed mean improves:
39.33$\rightarrow$77.00, 40.33$\rightarrow$80.33, and
41.00$\rightarrow$75.33. Thus the Direct gain survives both unseen expert
episodes and independent optimization runs rather than retrieving a training
episode.

\begin{table}[t]
\centering
\small
\caption{PushT successes on held-out expert episodes (100 per cell).}
\label{tab:episode_disjoint_cells}
\begin{tabular}{lrrrrrr}
\toprule
& \multicolumn{3}{c}{Transition inverse} & \multicolumn{3}{c}{Goal intent} \\
\cmidrule(lr){2-4}\cmidrule(lr){5-7}
Training seed & e0 & e1 & e42 & e0 & e1 & e42 \\
\midrule
3072 & 35 & 42 & 41 & 74 & 83 & 74 \\
3073 & 36 & 44 & 41 & 81 & 81 & 79 \\
3074 & 36 & 47 & 40 & 71 & 80 & 75 \\
\bottomrule
\end{tabular}
\end{table}

\section{Evaluation Protocol Scope}
\label{app:clear_scope}

The official LeWM benchmark remains the headline protocol because it preserves
comparability with prior work. CLEAR-LeWM is a separate, versioned, MIT-licensed
software benchmark rather than a contemporaneously reviewed paper; we cite the
public v0.5.1 release and immutable commit as a software artifact
\citep{sun2026clearlewm}. Official and CLEAR scores are never pooled. The
INTACT audit pins v0.5.1 commit
\texttt{32f4416c}, canonical manifests, batch
one, Math-only SDPA, deterministic algorithms, full checkpoint hashes, and
complete episode outcomes.

\subsection{Why a corrected protocol is informative}
The historical sampler draws rows uniformly and retains initially successful
start--goal pairs. This creates both episode-length bias and automatic-success
mass. Across the released datasets, 38.38\% of valid Cube pairs and 8.82\% of
TwoRoom $+25$ pairs are already within the historical endpoint threshold.
Random control consequently reaches 49\% and 26\%, respectively, in the
three-run historical audit. Moderate removes initially solved pairs, samples
episodes uniformly, restores the final valid row, and fixes only demonstrated
evaluator defects. This lowers the Cube and TwoRoom random floors to 15.67\%
and 6.67\% while retaining strong official-checkpoint performance
(\cref{tab:clear_floor}). A lower floor increases usable dynamic range; it does
not guarantee that every model score must decrease because the sampled pair
distribution also changes.

\begin{table}[t]
\centering
\scriptsize
\caption{Historical automatic-success audit and CLEAR Moderate calibration.
Historical random is the mean of three runs under the released stack. CLEAR
values use canonical seeds 0/1/42 and 100 episodes per seed.}
\label{tab:clear_floor}
\begin{tabular}{lrrr}
\toprule
Task & Initially solved & Historical random & Moderate random \\
\midrule
PushT & 0.22 & 4.00 & 4.00 \\
Cube & 38.38 & 49.00 & 15.67 \\
Reacher & 0.57 & 13.00 & 4.33 \\
TwoRoom & 8.82 & 26.00 & 6.67 \\
\bottomrule
\end{tabular}
\end{table}

\subsection{Moderate and Strict are different claims}
Moderate is the closest corrected continuation of the released task. Strict
builds on the same data hygiene and physics but asks for precise task-semantic
completion. They are reported separately rather than averaged
(\cref{tab:clear_contract}).

\begin{table}[t]
\centering
\tiny
\caption{CLEAR-LeWM v0.5.1 task contracts. ``Hold'' is the number of consecutive
successful environment steps. The 24-fold Cube comparison minimizes orientation
error over the proper rotational symmetry group of a cube.}
\label{tab:clear_contract}
\setlength{\tabcolsep}{2pt}
\begin{tabular}{p{0.12\columnwidth}p{0.38\columnwidth}p{0.40\columnwidth}}
\toprule
Task & Moderate: minimal compatibility correction & Strict: task-semantic precision \\
\midrule
PushT & Released pusher + T position $<20$ px and T angle $<20^\circ$; first hit. & T object only, $<10$ px and $<10^\circ$; hold 3. \\
Cube & Initially unsolved, cube-center distance $\leq4$ cm; first hit. & Cube center $\leq3$ cm and symmetry-aware orientation $\leq15^\circ$; hold 3. \\
Reacher & Periodic unbounded shoulder, raw bounded wrist, joint error $<0.05$ rad; first hit. & Physical fingertip endpoint $\leq1$ cm; hold 2. \\
TwoRoom & Clean cross-room pair, continuous swept-disk collision, endpoint $<16$ px. & Legal doorway crossing, valid route, goal-side arrival, endpoint $<8$ px. \\
\bottomrule
\end{tabular}
\end{table}

The task-specific repairs matter. Reacher's released all-periodic comparison
can wrap its physically bounded wrist coordinate; Moderate wraps only the
unbounded shoulder. The rewritten TwoRoom runtime previously checked requested
segment endpoints and could admit motion through a wall. Both modes now use
continuous swept-disk collision for the full agent radius; Strict additionally
requires a legal cross-room route and goal-side arrival. Cube Moderate retains
the OGBench 4 cm position task so that compatibility is explicit, whereas
Strict adds object orientation modulo 24 proper cube rotations. PushT Moderate
retains the released full pusher--object goal state; Strict removes the
task-irrelevant terminal pusher pose and scores precise T-object completion.

\subsection{Pinned references and public checkpoint bundles}
\Cref{tab:clear_public} summarizes the publicly auditable reference set;
INTACT has completed Moderate only, so no Strict score is imputed.

\begin{table}[!ht]
\centering
\scriptsize
\caption{Public CLEAR-LeWM v0.5.1 results. Each task cell is Moderate/Strict SR
(\%). INTACT uses audited Direct inference; official LeWM uses pinned CEM
$300{\times}30$. A dash denotes an unsubmitted mode. INTACT Moderate averages
three training seeds, each over canonical evaluation seeds 0/1/42; official
LeWM and random average those three evaluation seeds; community rows are
self-reported CEM seed-42 bundles and did not submit Reacher.}
\label{tab:clear_public}
\begin{adjustbox}{max width=\columnwidth}
\begin{tabular}{lrrrr}
\toprule
Model & PushT & Cube & Reacher & TwoRoom \\
\midrule
INTACT goal, Direct & 88.22/-- & 99.89/-- & 49.56/-- & 95.67/-- \\
Official LeWM & 86.33/70.33 & 50.33/26.33 & 46.00/43.00 & 84.00/58.33 \\
Paired random & 4.00/5.00 & 15.67/6.00 & 4.33/5.00 & 6.67/1.67 \\
DINOv2 No-Proprio & 8/7 & 43/17 & --/-- & 55/26 \\
GCBC Joint & 9/9 & 16/3 & --/-- & 15/9 \\
\bottomrule
\end{tabular}
\end{adjustbox}
\end{table}

The comparison illustrates the division of labor between protocols. The
official benchmark supplies historical continuity. Moderate removes automatic
success and known implementation defects while changing as little task
semantics as possible. Strict supports a stronger physical-completion claim.
For single-task INTACT Direct, Moderate preserves high PushT/Cube/TwoRoom
performance but changes Reacher from 97.67\% official SR to 49.56\%; this is exactly the kind of
near-ceiling ambiguity that the corrected topology is intended to reveal.

\end{document}